\documentclass[11pt,reqno,a4paper]{article}
\usepackage{fullpage}

\usepackage[utf8]{inputenc}
\usepackage[T1]{fontenc}
\usepackage{booktabs}
\usepackage{amsfonts}
\usepackage{nicefrac}
\usepackage[protrusion=true,expansion=false]{microtype}
\usepackage{xcolor}
\usepackage{bm}
\usepackage{amsmath,amssymb,amsbsy,amsfonts,amsthm}

\newtheorem{theorem}{Theorem}
\newtheorem{lemma}[theorem]{Lemma}
\newtheorem{definition}[theorem]{Definition}
\newtheorem{proposition}[theorem]{Proposition}
\newtheorem*{problem}{Question}

\usepackage{subcaption}
\usepackage{multirow}
\usepackage{makecell}
\usepackage{graphicx}
\usepackage{mathtools}
\usepackage{environ}
\usepackage[round]{natbib}

\usepackage{tablefootnote}
\usepackage{enumitem}
\usepackage{float}
\usepackage{afterpage}
\usepackage{placeins}

\usepackage{hyperref}

\usepackage{amsmath,amsfonts,bm}

\def\1{\bm{1}}

\def\vh{{\bm{h}}}

\def\vx{{\bm{x}}}

\def\mS{{\bm{S}}}

\def\mW{{\bm{W}}}
\def\mX{{\bm{X}}}
\def\mY{{\bm{Y}}}
\def\mZ{{\bm{Z}}}

\DeclareMathAlphabet{\mathsfit}{\encodingdefault}{\sfdefault}{m}{sl}
\SetMathAlphabet{\mathsfit}{bold}{\encodingdefault}{\sfdefault}{bx}{n}
\newcommand{\tens}[1]{\bm{\mathsfit{#1}}}

\def\tH{{\tens{H}}}

\def\tW{{\tens{W}}}
\def\tX{{\tens{X}}}

\def\tZ{{\tens{Z}}}

\usepackage{xparse}

\newcommand{\tXgglt}{\tX_g^{(\ell-1,t)}}

\newcommand{\tXflt}{\tX_f^{(\ell,t)}}

\NewDocumentCommand{\mXnlt}{ O{\ell} }{\ensuremath{\mX_n^{\left(#1,t\right)}}
}

\NewDocumentCommand{\mXnglt}{o}{%
  \ensuremath{%
    \mX_{n,g}^{\left(%
      \IfValueTF{#1}{#1}{\ell}, t%
    \right)}%
  }%
}

\NewDocumentCommand{\mYnlt}{ O{\ell} }{\ensuremath{\mY_n^{\left(#1,t\right)}}
}

\newcommand{\tXlt}{\tX^{(\ell,t)}}

\newcommand{\GNN}{\revised{\mathrm{GNN}}}
\newcommand{\GraphonNN}{\revised{\mathrm{WNN}}}

\newcommand{\norm}[1]{\left\lVert #1 \right\rVert}

\newcommand{\parens}[1]{\left( #1 \right)}

\newcommand{\revised}[1]{{\color{black}#1}}

\newcommand{\newrevised}[1]{{\color{black}#1}}

\NewDocumentCommand{\hltfgk}{ O{t} }{%
\ensuremath{\vh_{fgk}^{\left(\ell,#1\right)}}%
}

\newcommand{\supT}{\sup_{t\in[0,T]}}

\newcommand{\teachX}[1]{{\bm{X}^*_{#1}}}
\newcommand{\gndeX}[1]{{\bm{X}_{#1}}}

\title{\bf On the Convergence and Size Transferability of Continuous-depth Graph Neural Networks}
\author{Mingsong Yan\thanks{Department of Mathematics, University of California, Santa Barbara, CA. (\textit{mingsongyan@ucsb.edu})}, \quad
Charles Kulick\thanks{Department of Mathematics, University of California, Santa Barbara, CA. (\textit{charleskulick@ucsb.edu})}, \quad and
Sui Tang\thanks{Department of Mathematics, University of California, Santa Barbara, CA. (\textit{suitang@ucsb.edu})}}
\date{}

\begin{document}

\maketitle

\begin{abstract}
Continuous-depth graph neural networks, also known as Graph Neural Differential Equations (GNDEs), combine the structural inductive bias of Graph Neural Networks (GNNs) with the continuous-depth architecture of Neural ODEs, offering a scalable and principled framework for modeling dynamics on graphs. In this paper, we present a rigorous convergence analysis of GNDEs with time-varying parameters in the infinite-node limit, providing theoretical insights into their \textit{size transferability}. To this end, we introduce Graphon Neural Differential Equations (Graphon-NDEs) as the infinite-node limit of GNDEs and establish their well-posedness. Leveraging tools from graphon theory and dynamical systems, we prove the \textit{trajectory-wise} convergence of GNDE solutions to Graphon-NDE solutions. Moreover, we derive explicit convergence rates under two deterministic graph sampling regimes: (1) weighted graphs sampled from smooth graphons, and (2) unweighted graphs sampled from $\{0,1\}$-valued (discontinuous) graphons. We further establish size transferability bounds, providing theoretical justification for the practical strategy of transferring GNDE models trained on moderate-sized graphs to larger, structurally similar graphs without retraining. Numerical experiments using synthetic and real data support our theoretical findings.
\end{abstract}

\section{Introduction}

Graph Neural Networks (GNNs) \citep{scarselli2008graph} have achieved remarkable success in addressing graph-based machine learning tasks.  A practical attraction is their potential for \textit{size transferability}: a model trained on smaller graphs can often be deployed on larger, structurally similar graphs while maintaining competitive performance \citep{ruiz2020graphon,levie2021transferability}, thanks to local message passing and shared weights. This property is especially valuable as it helps avoid the substantial computational cost of retraining on large-scale graphs.

Size transferability does not arise unconditionally \citep{jegelka2022theory,cai2022convergence}. Recent theoretical advances justify size transferability through convergence analyses with conditions on graph sequences, message-passing operators, and activation functions. For example, graph sequences are assumed to be drawn from a shared generative model (i.e., graphons), which is common in complex network theory for modeling structurally similar graphs \citep{lovasz2012large}. Under appropriate assumptions, as graph size increases, the outputs of GNNs converge to a well-defined continuous limit, and specific convergence rates can be established. Notably, the convergence rates allow us to explicitly quantify the size transferability error between structurally similar graphs of varying sizes. Such results have recently been formalized for popular GNN architectures, including spectral, message-passing, invariant, and higher-order GNN networks \citep{ruiz2020graphon,keriven2020convergence,kenlay2021stability,levie2021transferability,cai2022convergence,maskey2023transferability,cordonnier2023convergence,le2024limits,herbst2025higher}.

Although the existing literatures primarily focus on discrete-layer GNN architectures, motivated by the growing interest in AI-for-Science applications, continuous-depth GNN models, often referred \revised{to} as Graph Neural Differential Equations (GNDEs), have recently gained attention \citep{poli2019graph,liu2025graph}. GNDEs model node features as solutions of an ordinary differential equation (ODE) parameterized by a GNN, which combine the expressivity of Neural ODEs \citep{chen2018neural} with the structural inductive biases inherent in GNNs. Distinct from discrete-layer architectures, GNDEs naturally incorporate \textit{time-varying} parameters, allowing the model to capture evolving node interactions and dynamic message-passing patterns over continuous time. Empirically, GNDEs have demonstrated strong performance, \revised{often matching or outperforming standard finite-depth GNNs without residual connections}, across a range of static and dynamic graph tasks, including node classification and link prediction \citep{poli2019graph,chamberlain2021grand,rusch2022graph,lin2024graph}, as well as practical applications such as epidemic forecasting \citep{luo2023hope,huang2024causal}, traffic prediction \citep{poli2019graph,choi2022graph,song2024traffic}, physical simulations \citep{han2022learning,huang2023ggode}, and recommendation systems \citep{xu2023graph}.

Despite their promise, GNDEs suffer from a crucial limitation of \emph{scalability}. This limitation arises because solving ODEs on large-scale graphs is computationally prohibitive \citep{finzi2023stable,liu2025graph}. To address this issue, we pursue a potential remedy via size transferability, which motivates the following question:

\begin{problem}[Size Transferability of GNDEs]\label{prob:gnne-transfer}
Do GNDEs exhibit size transferability? More specifically, if a GNN architecture with known size transferability is used to parameterize a Neural ODE, does the resulting continuous-depth model (i.e., GNDE) inherit this transferability property?
\end{problem}

\revised{Building on prior work in GNNs, we study the size transferability of GNDEs by analyzing the convergence of their solutions as the input graphs grow and converge to a limiting graphon.} Notably, GNDEs necessitate a stronger notion of  convergence than GNNs in terms of the infinite-node limit. For standard GNNs, the discrete-layer architecture generates only \emph{finitely} many hidden states (i.e., layer-wise outputs) during forward propagation. In contrast, the continuous-depth structure of GNDEs evolves node features through \emph{infinitely} many intermediate states, forming a trajectory over a continuous time horizon. Existing convergence results for GNNs \citep{ruiz2020graphon,maskey2023transferability} indicate that each hidden state would converge to a limiting function as the graph size grows. For GNDEs, an analogous property is expected: the entire feature trajectory should \emph{uniformly} converge as the number of nodes tends to infinity (cf.~Figure \ref{fig:trajectoryconvergence}). This \emph{trajectory-wise} convergence is crucial for both \emph{forward} and \emph{backward} propagation. 
\begin{itemize}
    \item During forward propagation, trajectory-wise convergence ensures that the continuous-time evolution of node features remains stable and consistent with increasing graph size. This is important for time-sensitive tasks such as forecasting and control, where intermediate states directly inform predictions and decisions. 

\item During backward propagation, trajectory-wise convergence provides \revised{a necessary prerequisite (though not by itself sufficient)}  for controlling the accumulation of approximation errors along the continuous-time trajectory, \revised{which underlies gradient-level stability considerations in continuous-depth models.}
\end{itemize}
Although trajectory-wise convergence for GNDEs is preferable, it \emph{cannot} be obtained by simply discretizing the dynamics (e.g., Euler's method) into a deep GNN with residual connections and invoking existing GNN results. This limitation arises for two main reasons. \textbf{(i)} It is challenging to select a uniform step size that remains suitable as the graph size increases, which prevents the discrete solution from accurately approximating the continuous-time graphon limit.
\textbf{(ii)} Evaluating the solution only at discrete time points overlooks the accumulation of error between those points and provides no guarantees on the temporal regularity of the solution. In light of the above issues, we perform the analysis directly in continuous time to obtain \emph{simultaneous, uniform-in-time} control of the entire trajectory as the graph size increases. This enables the use of techniques from dynamical systems, including stability estimates derived from Grönwall-type inequalities, \revised{which are more naturally obtained in the continuous-time setting than through purely discrete-time analyses.} %

\begin{figure}[!htbp]
  \centering
  \includegraphics[width=0.8\linewidth]{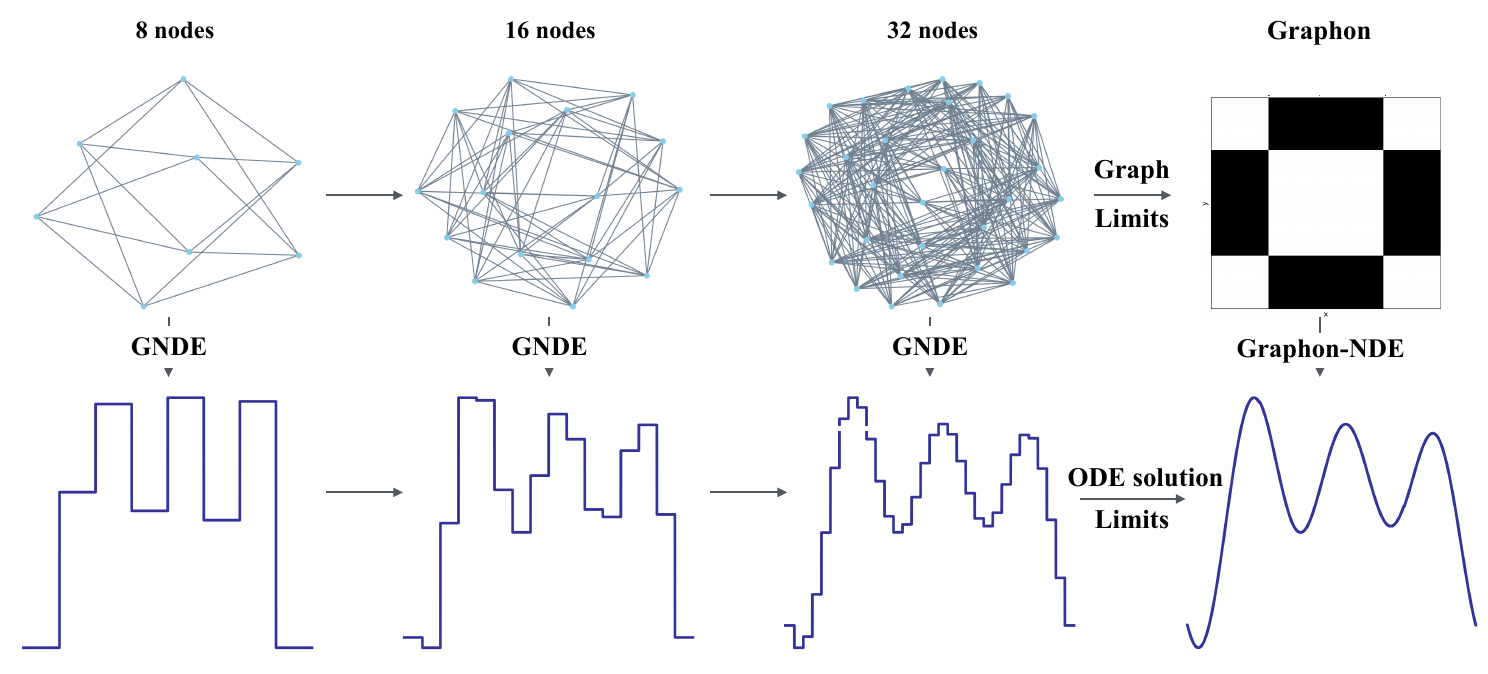} %
  \caption{\revised{Convergence of GNDE solutions (bottom) on graphon-sampled graphs (top) via induced graphon representations}. \newrevised{Discrete GNDE solutions are visualized using piecewise-constant embeddings. The bottom row shows the values of a scalar node feature at a fixed time, plotted against nodes ordered according to their latent graphon positions (or the continuum coordinate $u\in[0,1]$ in the graphon limit). Convergence to the smooth Graphon-NDE solution is observed as the graph size increases.}}
  \label{fig:trajectoryconvergence}
\end{figure}

\revised{\paragraph{Overview of the results} Analytically, we draw inspiration from the rich literature on deriving graphon limits for nonlocal diffusion equations
\citep{medvedev2014nonlinear,ayi2021mean,paul2022microscopic,kaliuzhnyi2022sparse}.
In particular, we adapt the stability-based graph-limit framework established by the seminal work of \citet{medvedev2014nonlinear}, originally developed for nonlinear heat equations, to the context of continuous-depth GNDEs parameterized by spectral Graph Convolutional Networks (GCNs).
Our theoretical results are summarized as follows:}

\begin{itemize} 

\item \textbf{Infinite-node limits (Graphon-NDEs).} We introduce an infinite-node limit of GNDEs, termed \textit{Graphon Neural Differential Equations} (Graphon-NDEs), \revised{which are a class of nonlocal integro-differential equations  defined on graphon spaces}.
\revised {We establish the well-posedness by identifying a set of explicit and architecture-driven regularity conditions on the activation functions and learnable filters. These assumptions effectively translate the kernel regularity requirements of the classical nonlocal diffusion framework \citep{medvedev2014nonlinear} into concrete, verifiable constraints on the neural architecture.}

    \item  \textbf{Trajectory-wise convergence.} We prove that solution trajectories of GNDEs (a sequence of ODEs) uniformly converge to \revised{the solution trajectory of} a Graphon-NDE (a PDE) whenever the underlying graph sequences and initial features converge, as illustrated in Figure \ref{fig:trajectoryconvergence}. \revised{Crucially, our analysis accommodates \emph{time-varying} filter parameters, a capability that enhances the expressivity of GNDEs over discrete GNNs. We establish that, even with this added complexity (i.e., nonautonomous dynamics), the entire inference trajectory remains stable and consistent as the graph size grows.}

 \item \textbf{Convergence rates.} We derive explicit convergence rates for both weighted and unweighted graphs which are generated deterministically from graphons. For weighted graphs sampled from smooth graphons, we present a convergence rate of \(\mathcal{O}(1/n^\alpha)\) with H\"older smoothness exponent $\alpha\in(0,1]$; for unweighted graphs sampled from $\{0,1\}$-valued (hence discontinuous) graphons, we \revised{leverage the box-counting dimension analysis established by \citet{medvedev2014nonlinear} to} show a convergence rate of \(\mathcal{O}(1/n^c)\), with \(c \in (0,1)\) depending on the box-counting dimension of the boundary of the graphon's support. \revised{In Section \ref{sec:graphon_conv_rate_experiment}, we provide numerical evidence consistent with these rates on  representative graphons.}

    \item \textbf{Size transferability bounds.} Leveraging our derived convergence rates, we establish upper bounds on the solution discrepancy of GNDEs over graphs of different sizes. This provides theoretical justification for the size transferability of GNDEs—models trained on smaller graphs can reliably generalize to larger, structurally similar graphs without retraining.

\end{itemize}

\subsection{Novelty and positioning to related work}

\revised{
Our results adapt the analytical techniques of \citet{medvedev2014nonlinear,kaliuzhnyi2022sparse} to the complex setting of vector-valued, nonautonomous neural dynamics. Following \citep{medvedev2014nonlinear}, we establish well-posedness via fixed-point arguments and trajectory-wise stability estimates through Gr\"onwall inequalities. Our contribution lies in identifying the proper conditions on the neural architecture, such as specific constraints on activation functions and time-evolving filters, to align with this classical stability framework. Technically, our stability estimates retain the structural form of classical nonlocal diffusion bounds \citep{medvedev2014nonlinear,kaliuzhnyi2022sparse}, where the solution discrepancy is controlled by the initial condition mismatch and the graphon approximation error. The distinction in the GNDE setting lies in the coefficients of this bound, which explicitly rely on the neural architecture, including activation functions, filters, network depth and convolutional channels. Finally, while we adopt H\"older assumptions on the graphon and initial features for a clear first presentation, these regularity assumptions enter only through the approximation terms, making the bound structurally compatible with the $L^{p}$-modulus-of-continuity (e.g., $\mathrm{Lip}(\alpha,L^{p})$) framework of \citet{kaliuzhnyi2022sparse}; We leave extensions to more general graphon classes and more general GNN architectures for future work.

To bridge theory with practice, Section~\ref{appendix: real data experiments} investigates these stability predictions on real-world node classification benchmarks using node subsampling. While real-world networks are finite and may not originate from a strict graphon limit, we leverage the induced graphon embedding as a canonical {scaling representation} that allows graphs of differing sizes to be compared within a common functional framework.
Across ten datasets spanning both homophilic and heterophilic regimes, we observe that the empirical performance gap under size transfer is strongly correlated with this theoretically motivated induced-graphon discrepancy in our stability estimates.

}

\paragraph{Error Decomposition for Zero-Shot Size Transfer in Learning Nonlocal Diffusion Dynamics}

\revised{
Nonlocal diffusion equations are widely used to model collective dynamics in physical and biological systems, and GNDEs have shown strong empirical performance as data-driven surrogates for such dynamics, including multi-particle and network traffic dynamics \citep{poli2019graph}, as well as opinion dynamics on graphs \citep{berndt2025permutation}.  From a data science perspective, this work provides a theoretical foundation for using GNDEs as data-driven surrogate models for learning diffusion-like dynamics with \textit{zero-shot size transferability}, \newrevised{the ability to deploy a model trained on graphs of one size to graphs of different sizes without retraining.}

Formally, let $\mathcal{G}$ be a graph sampled from a graphon, and let $S_{\mathcal{G}}$ denote the solution of a target discrete dynamical system posed on $\mathcal{G}$ (e.g., a nonlocal diffusion ODE).
We consider a GNDE model $\Phi_{\mathcal{G}}(\theta)$ trained to approximate $S_{\mathcal{G}}$, where $\theta$ denotes the learned parameters.
To reduce training cost, we train a surrogate GNDE model on a smaller graph $\mathcal{G}_{\mathrm{small}}$ sampled from the same graphon, yielding parameters $\hat{\theta}$, and then evaluate the resulting model in a zero-shot manner on the larger graph $\mathcal{G}$.
By lifting all quantities to a common function space via induced graphon representations, the total prediction error admits the following decomposition by the triangle inequality:
\begin{equation}\label{eq:triangle_decomposition}
\begin{aligned}
\bigl\| \Phi_{\mathcal{G}}(\hat{\theta}) - S_{\mathcal{G}} \bigr\|
\;\le\;&
\underbrace{\bigl\| \Phi_{\mathcal{G}}(\hat{\theta}) - \Phi_{\mathcal{G}_{\mathrm{small}}}(\hat{\theta}) \bigr\|}_{\text{(I) model transfer error}}
\;+\;
\underbrace{\bigl\| \Phi_{\mathcal{G}_{\mathrm{small}}}(\hat{\theta}) - S_{\mathcal{G}_{\mathrm{small}}} \bigr\|}_{\text{(II) training error}}
\;+\;
\underbrace{\bigl\| S_{\mathcal{G}_{\mathrm{small}}} - S_{\mathcal{G}} \bigr\|}_{\text{(III) graph discretization error}}.
\end{aligned}
\end{equation}
Term (III) represents \newrevised{the} \textit{graph discretization error}, which captures the error induced by approximating the underlying physical law via a finite graph and is governed by classical graph-limit theory \citep{medvedev2014nonlinear}.
Term (II) corresponds to the \textit{training error} on the source graph; In Table~\ref{tab:heat_results_num} in Section~\ref{appendix: Synthetic Dynamics Experiments} and Figure~\ref{fig:three_panel} in Section~\ref{appendix: Synthetic Dynamics Experiments}, we show that this term is small for four representative nonlocal diffusion dynamics on graphs sampled from five underlying graphons under a standardized training protocol. Term (I), the \emph{model transfer error}, is the primary focus of the present work. Our analysis shows that this term is controlled by the underlying graphon approximation and initial feature mismatch. 

Taken together, the above observations support continuous-depth spectral GNDEs as a principled surrogate class for learning diffusion-like dynamics across graph scales, and help bridge graph-limit theory with practical transfer learning for graph-based dynamical models.

}

\paragraph{Additional Related Works} 

Convergence theory in deep learning includes width-wise limits, where infinitely wide networks converge to kernel models via the Neural Tangent Kernel (NTK) \citep{jacot2018neural}, and depth-wise limits, where deep residual networks converge to Neural ODEs \citep{weinan2017proposal,chen2018neural,avelin2021neural,sander2022residual,thorpe2023deep}. In contrast, our focus is on input-wise convergence: \revised{we analyze whether GNDEs produce stable outputs as the input graphs grow in size and converge to a limiting graphon}. Closely related works apply graphon theory and analyze infinite-node limits for GNNs and message-passing networks \citep{ruiz2020graphon,gama2020stability,keriven2020convergence,ruiz2021graphon,morency2021graphon,kenlay2021stability,kenlay2021interpretable,levie2021transferability,maskey2023transferability,cordonnier2023convergence,le2024limits}. We extend prior graphon-based convergence results for GNNs to continuous-depth models with time-varying parameters, establishing a stronger trajectory-wise convergence. A summary comparison is provided in Table~\ref{tab:comparison}.

\afterpage{
\begin{table}
\centering
\small
\renewcommand{\arraystretch}{1.4} %
\setlength{\tabcolsep}{6pt} %
\begin{tabular}{|m{2.8cm}|m{4.7cm}|m{4.5cm}|}
\hline
\multicolumn{1}{|c|}{\textbf{Attribute}} & \multicolumn{1}{c|}{\textbf{GNNs}} & \multicolumn{1}{c|}{\textbf{GNDEs (ours)}} \\
\hline
\multicolumn{1}{|c|}{Layer Type} & \multicolumn{1}{c|}{Discrete} & \multicolumn{1}{c|}{Continuum} \\
\hline
\multicolumn{1}{|c|}{Coefficient Type} & \multicolumn{1}{c|}{Static} & \multicolumn{1}{c|}{Temporally Continuous} \\
\hline
\multicolumn{1}{|c|}{Convergence Notion} & \multicolumn{1}{c|}{Layer-wise}  & \multicolumn{1}{c|}{Trajectory-wise} \\
\hline
\multicolumn{1}{|c|}{\textbf{Graphon Type}} & \multicolumn{2}{c|}{\textbf{Convergence Rates}} \\
\hline
\multicolumn{1}{|c|}{Smooth} & \multicolumn{1}{c|}{\makecell[t]{\rule{0pt}{6pt}Lipschitz Continuous\\$\mathcal{O}(1/\sqrt{n})$ \citep{ruiz2020graphon}; $\mathcal{O}(1/n)$ \citep{maskey2023transferability}}} & \multicolumn{1}{c|}{\makecell[t]{\rule{0pt}{6pt}H\"older Continuous\\$\mathcal{O}(1/n^\alpha)$, $\alpha\in(0,1]$}}\\
\hline
\multicolumn{1}{|c|}{$\{0,1\}$-valued} & \multicolumn{1}{c|}{Inexplicit\footnotemark~\citep{morency2021graphon,kenlay2021stability,kenlay2021interpretable}} & \multicolumn{1}{c|}{\textbf{$\mathcal{O}(1/n^c)$}, $c \in (0,1)$}\\
\hline
\end{tabular}
\vspace{6pt}
\caption{Infinite-node Convergence Rates for GNNs and GNDEs}\label{tab:comparison}
\end{table}
\footnotetext{``Inexplicit” means that convergence is established, but no explicit rate is provided.}
}

\section{Notation and Preliminary Concepts}
Let $\mathbb{N}:=\{1,2,\ldots\}$ and $\mathbb{R}^+:=[0,\infty)$. For $n\in\mathbb{N}$, let $[n]:=\{1,2,\ldots,n\}$, $\mathbb{Z}_n:=\{0,1,\ldots,n-1\}$. The norm $\|\cdot\|_2$ is \revised{the} Euclidean norm for vectors and \revised{the} spectral norm for matrices. We denote the unit interval as \( I := [0, 1] \) and $I^2:=I\times I$. For an interval \( J \subseteq I \), by $|J|$ we denote the length of $J$, and we define the indicator function \( \chi_J:J\to\{0,1\} \) as $\chi_{J}(u):=1$ if $u \in J$, and $\chi_{J}(u):=0$ otherwise. For a finite set $S$, by $|S|$ we denote the number of elements in the set $S$. 

\paragraph{Function spaces} The function space \( L^p(I; \mathbb{R}^{1\times F}) \) consists of all \( L^p \)-integrable vector valued functions mapping \( I \) to \( \mathbb{R}^{1\times F} \), where \( p\in[1,\infty] \) and \( F \) denotes the number of features. The norm in \( L^p(I; \mathbb{R}^{1\times F}) \) is defined by $\|\tZ\|_{L^p(I;\mathbb{R}^{1\times F})}:=(\sum_{f\in[F]}\|Z_f\|_{L^p(I)}^2)^{1/2}$ for $\tZ=[Z_f:f\in[F]]$. By $\int_I\tZ(u)du$ we denote the entry-wise integral $\left[\int_I Z_f(u)du:f\in [F]\right]$. Let $\Omega$ be a subset of $\mathbb{R}^+$ and $p\in[1,\infty]$, the Banach space \( C(\Omega; L^p(I; \mathbb{R}^{1\times F})) \) is composed of vector-valued functions \( \tX=[X_f:f\in[F]]: I \times \Omega  \to \mathbb{R}^{1\times F} \) satisfying that for each \( t \in \Omega \), \( \tX(\cdot,t )\in L^p(I; \mathbb{R}^{1\times F}) \); for each \( u \in I \) and $f\in[F]$, \( X_f(u,\cdot) \) is continuous on $\Omega$; and with finite norm $\|\tX\|_{C(\Omega;L^{p}(I; \mathbb{R}^{1\times F}))}:=\sup_{t\in\Omega}\|\tX(\cdot,t)\|_{L^{p}(I;\mathbb{R}^{1\times F})}$. By $C^1(\Omega;L^{p}(I; \mathbb{R}^{1\times F}))$ we denote a subspace of $C(\Omega;L^{p}(I; \mathbb{R}^{1\times F}))$, in which the vector-valued function $\tX=[X_f:f\in[F]]$ satisfies that for each $f\in[F]$ and $u\in I$, $X_f(u,\cdot)$ is continuously differentiable. 

\paragraph{Graph and graph features} A graph is denoted by \( \mathcal{G} = \langle V(\mathcal{G}), E(\mathcal{G}),\mW_{\mathcal{G}} \rangle \), where \( V(\mathcal{G}) \) is the set of nodes, \( E(\mathcal{G}) \subseteq V(\mathcal{G})\times V(\mathcal{G})\) is the set of edges, and $\mW_{\mathcal{G}} = \left[[\mW_{\mathcal{G}}]_{ij} \in I : i, j \in |V(\mathcal{G})|\right]$ represents the adjacency matrix, where $[\mW_{\mathcal{G}}]_{ij} = [\mW_{\mathcal{G}}]_{ji}$ is nonzero \revised{if and only if} $(i, j) \in E(\mathcal{G})$. %
We say a graph $\mathcal{G}$ is \textit{weighted} if the entries in its adjacency matrix $\mW_{\mathcal{G}}$ are real numbers in the unit interval $I$, and \textit{unweighted} if the entries in $\mW_{\mathcal{G}}$ are restricted to $\{0,1\}$, where $[\mW_{\mathcal{G}}]_{ij} = 0$ indicates the absence of an edge and $[\mW_{\mathcal{G}}]_{ij}=1$ indicates the presence of an edge. \revised{A graph $\mathcal{G}$ is \textit{simple} if it is undirected, and containing no self-loops or multiple edges.} Let $F$ be the number of features and by $\mZ_{\mathcal{G}}$ we denote the graph node feature matrix \( \mZ_\mathcal{G} \in \mathbb{R}^{|V(\mathcal{G})| \times F} \) over the graph $\mathcal{G}$, which assigns a feature vector \( [\mZ_\mathcal{G}]_{i,:} \in \mathbb{R}^{1 \times F} \) for each node \( i \in V(\mathcal{G}) \).

\paragraph{Graphon and graphon features} A \textit{graphon} is a bounded, symmetric, and measurable function $\tW: I^2 \rightarrow I$, serving as a continuous generalization of the discrete adjacency matrix. One can treat a graphon as an undirected graph with a continuum of nodes from the unit interval $I$, where the edge weight between nodes $u_i, u_j \in I$ is given by $\tW(u_i, u_j)$. A graphon can represent the \textit{limiting structure} of a sequence 
of finite graphs, with the formal definition of convergence deferred to Section \ref{section: graph limits}. This perspective enables graphons 
to model entire classes of graphs with similar connectivity patterns. A \textit{graphon feature function} $\tZ$ over a graphon $\tW$ is a function $\tZ:I\to\mathbb{R}^{1\times F}$, where $\tZ(u)$ represents the node features for each $u\in I$. Similarly, graphon feature functions can be viewed as a generalization of discrete graph node feature matrices over continuum nodes.

\subsection{Graph Limits}\label{section: graph limits}
In this section, we provide more details of graphons as graph limits. We begin with the concept of a sequence of graphs converging to a graphon in the sense of \textit{homomorphism density} \citep{lovasz2012large}. A \textit{motif} $\mathcal{F}$ is an arbitrary simple graph. A \textit{homomorphism} from a motif $\mathcal{F}$ to a simple unweighted graph $\mathcal{G}$ is an adjacency-preserving mapping $\phi: V(\mathcal{F}) \rightarrow V(\mathcal{G})$, meaning $(i, j) \in E(\mathcal{F})$ implies $(\phi(i), \phi(j)) \in E(\mathcal{G})$, and the \textit{homomorphism number} $\text{hom}(\mathcal{F},\mathcal{G})$ refers to the total number of homomorphisms from $\mathcal{F}$ to $\mathcal{G}$. The \textit{homomorphism density} $t(\mathcal{F},\mathcal{G})$ is defined as the ratio of $\text{hom}(\mathcal{F}, \mathcal{G})$ and $|V(\mathcal{G})|^{|V(\mathcal{F})|}$, which represents the probability of a random mapping $\phi: V(\mathcal{F}) \rightarrow V(\mathcal{G})$ being a homomorphism. The notion of homomorphism density can be similarly extended to the case of $\mathcal{G}$ being \revised{a weighted graph} \citep{lovasz2012large}. \revised{For a weighted graph $\mathcal{G}$, the homomorphism density preserves the form $t(\mathcal{F}, \mathcal{G}):=\text{hom}(\mathcal{F}, \mathcal{G})/|V(\mathcal{G})|^{|V(\mathcal{F})|}$, but with the homomorphism number generalized for weighted edges as $\text{hom}(\mathcal{F}, \mathcal{G}) := \sum_{\phi} \prod_{(i,j) \in E(\mathcal{F})} [\mathbf{W}_{\mathcal{G}}]_{\phi(i)\phi(j)}$, where the sum runs over all maps $\phi:V(\mathcal{F})\to V(\mathcal{G})$.}

The homomorphism density from a motif to a graphon is generalized via integrals. We define the homomorphism density from a motif \(\mathcal{F}\) to a graphon \(\tW\), denoted by \( t(\mathcal{F}, \tW) \), as 
\begin{equation}\label{graphonhormomorphismdensity}
t(\mathcal{F},\tW):=\int_{I^{|V(\mathcal{F})|}}\prod_{(i,j)\in E(\mathcal{F})}\tW(u_i,u_j)\prod_{i\in V(\mathcal{F})}du_i.
\end{equation}
We say that a sequence of graphs \(\{\mathcal{G}_n\}_{n=1}^\infty\) converges to the graphon \(\tW\) in the sense of homomorphism density if, for any motif \(\mathcal{F}\), it holds that
\[
\lim_{n \to \infty} t(\mathcal{F}, \mathcal{G}_n) = t(\mathcal{F}, \tW).
\]
In the sense of homomorphism density, every graphon is a limit object of some convergent graph sequence, and conversely, every convergent graph sequence converges to a unique graphon \citep{lovasz2012large}. Thus, a graphon represents a family of graphs that approximate a same underlying structure, even if their sizes differ. By categorizing graphs into such ``graphon families'', graphons allow for easier and more structured analysis of graph sequences, providing a robust framework for studying large-scale networks.

In the following, we review the relation of convergence in homomorphism density and cut norm. The \textit{cut norm} of a graphon $\tW$ is defined by 
\begin{align}\label{limitshormomorphismdensity}
    \|\tW\|_{\square} := \sup_{S, S' \subseteq I} \left| \int_{S \times S'} \tW(x, y) \, dx \, dy \right|,
\end{align}
where the supremum is taken over all subsets $S$ and $S'$ of $I$. \revised{The cut norm measures the maximum connection strength between any two subsets of nodes, serving as a metric to quantify the structural difference between graphons.} Let $\mathcal{G}$ be a graph with adjacency matrix $\mW_{\mathcal{G}}\in\mathbb{R}^{|V(\mathcal{G})|\times|V(\mathcal{G})|}$. For each $i \in [|V(\mathcal{G})|]$, let $I_i := \left[\frac{i-1}{|V(\mathcal{G})|}, \frac{i}{|V(\mathcal{G})|}\right)$. The induced graphon representation of $\mW_{\mathcal{G}}$, denoted as $\tW_{\mathcal{G}} : I^2 \to \mathbb{R}$, is defined by
\begin{equation}\label{induced graphon Wn}
\tW_{\mathcal{G}}(u, v) := \sum_{i,j=1}^{|V(\mathcal{G})|} [\mW_{\mathcal{G}}]_{ij} \, \chi_{I_i}(u) \, \chi_{I_j}(v), \quad u,v \in I.
\end{equation}
\revised{
A permutation $\pi$ acts as a relabeling of the nodes of a graph. Formally, for a graph $\mathcal{G}$, a permutation $\pi$ is a bijection on $V(\mathcal{G})$. The permuted graph, denoted by $\pi(\mathcal{G})$, is defined as the graph with the same node set $V(\mathcal{G})$, edge set $E(\pi(\mathcal{G})):=\{(\pi(i),\pi(j)):(i,j)\in E(\mathcal{G})\}$, and adjacency matrix $\mW_{\pi(\mathcal{G})}$ given by
$
[\mW_{\pi(\mathcal{G})}]_{\pi(i),\pi(j)} := [\mW_{\mathcal{G}}]_{i,j}$, $i,j\in V(\mathcal{G})$.}
The following result from \citet{lovasz2012large} states that the convergence of graphs in terms of homomorphism density implies convergence in the cut norm of induced graphons, up to some permutations.
\begin{lemma}\label{lemma: convergence in homomorphism density implies cut norm}
    Let $\{\mathcal{G}_n\}_{n=1}^\infty$ be a sequence of graphs with adjacency matrices $\{\mW_{\mathcal{G}_n}\}_{n=1}^\infty$. Suppose that $\{\mathcal{G}_n\}_{n=1}^\infty$ converges to a graphon $\tW$ in the sense of homomorphism density. Then, there exists a sequence $\{\pi_n\}_{n=1}^\infty$ of permutations such that     $\lim_{n\to\infty}\|\tW_{\pi_n(\mathcal{G}_n)}-\tW\|_{\square}=0.$
\end{lemma}
Given a sequence $\{\mathcal{G}_n\}_{n=1}^\infty$ of graphs converging to a graphon $\tW$ in the sense of homomorphism density, we introduce a set of the permutation sequences $\{\pi_n\}_{n=1}^\infty$ such that the permuted induced graphons $\tW_{\pi_n(\mathcal{G}_n)}$ converge under the cut norm to the graphon $\tW$, that is, 
\begin{equation}\label{admissible set}
    \mathfrak{P}:=\left\{\{\pi_n\}_{n=1}^\infty:\lim_{n\to\infty}\|\tW_{\pi_n(\mathcal{G}_n)}-\tW\|_{\square}=0\right\}.
\end{equation}
It is clear that the set $\mathfrak{P}$ is not empty due to Lemma \ref{lemma: convergence in homomorphism density implies cut norm}. 

\subsection{Graph-feature limits} In the following, we formulate the convergence of a sequence of graph-feature pairs to a graphon-feature pair. To this end, we introduce the convergence of induced graphon feature functions. 

Let $\mathcal{G}$ be a graph with node feature matrix $\mZ_{\mathcal{G}} \in \mathbb{R}^{|V(\mathcal{G})| \times F}$. 
The induced graphon feature function $\tZ_{\mathcal{G}} : I \to \mathbb{R}^{1 \times F}$, defined as the piecewise constant interpolation of the node feature matrix $\mZ_{\mathcal{G}}$, is given by
\begin{equation}\label{induced graph features Zn}
\tZ_{\mathcal{G}}(u) := \sum_{i=1}^{|V(\mathcal{G})|} [\mZ_{\mathcal{G}}]_{i,:} \, \chi_{I_i}(u), \quad u \in I.
\end{equation}
We adopt the following definition of graph-feature pairs converging to a graphon-feature pair, introduced in \citet{ruiz2021graphon}.
\begin{definition}\label{def: graph node feature converges -- formal one}
    Let $\{\mathcal{G}_n\}_{n=1}^\infty$ be a sequence of graphs with adjacency matrices $\{\mW_{\mathcal{G}_n}\}_{n=1}^\infty$ and graph node feature matrices $\{\mZ_{\mathcal{G}_n}\}$. Suppose that $\{\mathcal{G}_n\}_{n=1}^\infty$ converges to a graphon $\tW$ in the sense of homomorphism density. Let $\tZ\in L^2(I;\mathbb{R}^{1\times F})$ be a graphon feature function. We say that $\{(\mathcal{G}_n,\mZ_{\mathcal{G}_n})\}_{n=1}^\infty$ converges to $(\tW,\tZ)$ if there exists a sequence of permutations $\{\pi_n\}_{n=1}^\infty\in\mathfrak{P}$ such that $\lim_{n\to\infty}\|\tZ_{\pi_n(\mathcal{G}_n)}-\tZ\|_{L^2(I;\mathbb{R}^{1\times F})}=0$, where the set $\mathfrak{P}$ is defined by \eqref{admissible set}. 
\end{definition}

\subsection{Graph Neural Differential Equations} %

\textit{Graph Neural Differential Equations} (GNDEs) \citep{poli2019graph} \newrevised{extend} Neural ODEs \citep{chen2018neural} to the graph domain by modeling the continuous-time dynamics with a Graph Neural Network (GNN). Formally, a GNDE is defined as
\begin{align}\label{GNDE}
\begin{split}
\frac{d}{dt} \mX(t) &= \GNN\big(\mX(t);\mS,\tH(t)\big), \\
\mX(0) &= \mZ \in \mathbb{R}^{|V(\mathcal{G})| \times F},
\end{split}
\end{align}
in which \(\mX(t) \in \mathbb{R}^{|V(\mathcal{G})| \times F}\) denotes the node feature matrix at time \(t\) and is initialized by the input node feature matrix \(\mZ\) at $t=0$; and \(\GNN\) is parameterized by a graph shift operator \(\mS\) and trainable, time-varying parameters \(\tH(t)\). \revised{We remark that the \textit{well-posedness} (i.e., the existence and uniqueness of solutions) of the GNDE system follows directly from standard theory of ODEs (e.g., the Picard–Lindel\"of theorem), under standard regularity conditions (formalized later as Assumptions~\hyperref[AS0]{AS0} and~\hyperref[AS1]{AS1}).}

\revised{While various GNN designs for the velocity field in \eqref{GNDE} have been proposed in the GNDE literature \citep{xhonneux2020continuous,chamberlain2021grand,rusch2022graph,choi2023gread}, we focus here on the case where it is parameterized by a spectral GCN.} Our analysis can similarly be applied to more general choices, which we leave for future work. The convergence analysis of GNNs in the infinite-node limit has been well studied in the literature \citep{ruiz2020graphon, krishnagopal2023graph, keriven2023functions, maskey2023transferability}, whereas our goal is to establish \emph{trajectory-wise} convergence for GNDEs, due to their fundamentally different continuous-depth architecture.

\paragraph{Graph Neural Networks} To ground our discussion, we first review the formulation of spectral GCNs. \revised{For a graph $\mathcal{G}$, a matrix \(\mS \in \mathbb{R}^{|V(\mathcal{G})|  \times |V(\mathcal{G})| }\) is called a \textit{graph shift operator} (GSO) if it satisfies \(\mS_{ij} = 0\) for $(i,j)\notin E(\mathcal{G})$ and \(i \neq j\)} \citep{shuman2013emerging,sandryhaila2013discrete, ortega2018graph}. Common choices for \(\mS\) include the normalized adjacency matrix or the graph Laplacian, both of which effectively capture the topological structure of the graph. \revised{In spectral GNNs, the notion of convolution from image-based Convolutional Neural Networks (CNNs) is generalized to graph-structured data using the GSO.} Let \(\vx \in \mathbb{R}^{|V(\mathcal{G})|} \) be a \textit{graph signal}, and let \(\vh = [h_0, h_1, \dots, h_{K-1}] \in \mathbb{R}^K\) be a \textit{filter}. \newrevised{Here, the value of $K$ is a model hyperparameter, which is given and fixed a priori.} The graph convolution of \(\vx\) with \(\vh\) is defined as $\sum_{k=0}^{K-1} h_k \mS^k \vx$, where \(\mS^k\) is the $k$-th power of $\mS$. \newrevised{This operation employs powers of the GSO to aggregate information from a node’s neighbors up to $K-1$ hops away \citep{bruna2013spectral,kipf2016semi}, which we refer to as $K$-hop aggregation.} 

Now, let $\tH = \{\vh_{fgk}^{(\ell)}\in\mathbb{R}: f, g \in [F], k\in\mathbb{Z}_K, \ell \in [L]\}$ \newrevised{be the set of filter parameters, which are trainable. Here,} $\vh_{fgk}^{(\ell)}$ is the $k$-th component of the filter used in the $\ell$-th layer to transform the $g$-th input feature into the $f$-th output feature. The $f$-th feature output of the $\ell$-th layer is computed by $\mX_f^{(\ell)} := \sigma( \sum_{g=1}^{F}\sum_{k=0}^{K-1}\vh_{fgk}^{(\ell)}\mS^k\mX_g^{(\ell-1)})$ where $\mX^{(0)}$ is the input feature matrix, $\sigma$ is a nonlinear activation function (e.g., ReLU). Then a GNN with $L$ layers can be compactly expressed as $\GNN(\mX^{(0)};\mS, \tH):=\mX^{(L)}$, which represents the overall GNN mapping from the input features $\mX^{(0)}$ to the output $\mX^{(L)}$, conditioned on the GSO $\mS$ and the filter parameters $\tH$.

In GNDEs, the velocity $d\mX/dt$ is evaluated as the GNN output, where we allow the filter parameters $\tH$ to vary over time. Specifically, we denote
\begin{equation}\label{parameterized H(t) filters}
\tH(t) := \left\{\vh_{fgk}^{(\ell,t)} : f, g \in [F],  k \in \mathbb{Z}_K, \ell \in [L] \right\}. 
\end{equation}
The setting of time-varying parameters enables the model to capture more complex dynamic processes on graphs.

\paragraph{Graphon Neural Networks} We briefly review the setup of \textit{Graphon Neural Networks} (Graphon-NNs) and one can refer to \citep{ruiz2020graphon} for more technical details. Given a graphon \(\tW: I^2 \to I\), the \textit{graphon integral operator}, denoted by \(T_{\tW}\), is defined for any feature function \(\mathbf{x} \in L^2(I; \mathbb{R})\) as $T_{\tW} \mathbf{x}(v) := \int_I \tW(u, v) \mathbf{x}(u) \, du$, $v\in I$. This operator is self-adjoint and Hilbert-Schmidt, with eigenvalues in \([-1, 1]\) accumulating around zero. For a filter \(\vh = [h_0,h_1,\dots, h_{K-1}]\), the \textit{graphon convolution} of $\mathbf{x}$ with $\vh$ is defined by $\sum_{k=0}^{K-1} h_k T_{\tW}^k \mathbf{x}$, where \(T_{\tW}^0\) is the identity operator and \(T_{\tW}^k\) is the \(k\)-fold composition of \(T_{\tW}\). Let \(\tX^{(0)} \in L^{\infty}(I; \mathbb{R}^{1 \times F})\) be the input feature function of a Graphon-NN. The \(f\)-th feature at the \(\ell\)-th layer of the Graphon-NN is updated via $\tX_f^{(\ell)} = \sigma( \sum_{g=1}^{F} \sum_{k=0}^{K-1} \vh_{fgk}^{(\ell)}T_{\tW}^k\tX_g^{(\ell-1)} )$, where \(\sigma\) is a nonlinear activation function. Then a Graphon-NN with \(L\) layers can be expressed as $\GraphonNN(\tX^{(0)};\tW, \tH):=\tX^{(L)}$, where $\GraphonNN$ represents the entire Graphon-NN mapping from the input feature function $\tX^{(0)}$ to the output feature function $\tX^{(L)}$, associated with graphon $\tW$ and parameters $\tH$.

\section{Main Results}

\subsection{Infinite-Node Limits: Graphon Neural Differential Equations and Well-Posedness} %
\label{section: Graphon-NDE}

To explore the infinite-node limiting structure of GNDEs, we introduce \textit{Graphon Neural Differential Equations} (Graphon-NDEs). Recalling that a graphon, as the limiting object of finite graphs, can be viewed as a graph with a continuum of nodes over the unit interval, we define Graphon-NDEs in a form similar to GNDEs \eqref{GNDE}, but tailored to operate on graphons rather than finite graphs. Specifically, we formulate Graphon-NDEs as 
\begin{align}\label{Graphon-NDE}
    \begin{split}
        \frac{\partial}{\partial t} \tX(u, t) &= \GraphonNN(\tX(u, t);\tW, \tH(t)), \\
        \tX(u, 0) &= \tZ(u),
    \end{split}
\end{align}
where \(\tX(\cdot, t):I\to\mathbb{R}^{1\times F}\) is the graphon node feature function at time $t$ and initialized by an input node feature function $\tZ$ at $t=0$; and \(\GraphonNN\) is a Graphon-NN applying on $\tX(\cdot,t)$ through graphon \(\tW\) and time-varying parameters $\tH(t)$ as in \eqref{parameterized H(t) filters}.

The continuum nature of both the node and time variables in Graphon-NDEs necessitates careful technical treatment to establish that they are well-posed. We prove that the temporal continuity of the filter evolution and the Lipschitz property of the activation function (Assumptions~\hyperref[AS0]{AS0} and~\hyperref[AS1]{AS1} below) suffice to guarantee well-posedness.
\begin{itemize}
    \item\label{AS0} \textbf{AS0.} The convolutional filters \revised{evolve} continuously in time, i.e., $\vh_{fgk}^{(\ell,t)}$ is a continuous function \revised{of} $t\in[0,T]$, for each $f,g\in [F]$, $\ell\in[L]$, $k\in\mathbb{Z}_K$. 
\end{itemize}
\newrevised{We remark that, in practice, the filter parameters $\vh_{fgk}^{(\ell,t)}$ may be chosen to be constant with respect to $t$, or parameterized as $\sum_i a_i \,\phi_i(t)$, where $\{\phi_i\}$ are fixed, continuous basis functions. In the second case, training would optimize the finite-dimensional coefficients $\{a_i\}$.}
\begin{itemize}
    \item\label{AS1} \textbf{AS1.} \revised{The activation function \(\sigma\) is  $L_\sigma$-Lipschitz continuous, i.e., $|\sigma(x) - \sigma(y)| \leq L_\sigma|x - y|$, for all $x, y \in \mathbb{R}$; and $\sigma(0) = 0.$}
\end{itemize}

\begin{theorem}[Well-posedness, proof in Appendix~\ref{Appendix: well-posedness}]\label{theorem: well-posedness}
Suppose that \hyperref[AS0]{AS0} and \hyperref[AS1]{AS1} hold. If \(\tW \in L^{\infty}(I^2)\) and \(\tZ \in L^{\infty}(I; \mathbb{R}^{1 \times F})\), then for any \(T > 0\), there exists a unique solution \(\tX \in C^1\left([0, T]; L^{\infty}(I; \mathbb{R}^{1 \times F})\right)\) to the Graphon-NDE \eqref{Graphon-NDE}.
\end{theorem}

\paragraph{Discussion} We remark that Assumptions \hyperref[AS0]{AS0} and \hyperref[AS1]{AS1} in Theorem \ref{theorem: well-posedness} are mild and are commonly satisfied in practical settings. GNDEs equipped with temporally continuous filters benefit from effective training methodologies, such as the Galerkin method, which represents filters as linear combinations of predefined continuous basis functions \citep{massaroli2020dissecting}. Another prominent class of GNDEs utilizes temporally piecewise constant filters \citep{massaroli2020dissecting}, which do not satisfy \hyperref[AS0]{AS0}. Nevertheless, our results remain applicable to individual time intervals, guaranteeing the existence of the graphon limit on each interval. Furthermore, standard activation functions, including ReLU, leaky ReLU, and tanh, adhere to \hyperref[AS1]{AS1} with \revised{$L_{\sigma}=1$} \citep{virmaux2018lipschitz}. \revised{
The proof of Theorem~\ref{theorem: well-posedness} follows a Banach fixed-point argument adapted from Theorem 3.2 in \citet{medvedev2014nonlinear}, where nonlinear heat equations on graph limits were studied. Our contribution is to adapt this argument to vector-valued dynamics induced by multi-layer spectral GCN architectures with time-dependent filters.}

The well-posedness result established in Theorem~\ref{theorem: well-posedness} paves the way for the subsequent convergence analysis of GNDE solutions to the Graphon-NDE solution as the sequence of structurally similar graphs converges to a graphon. Theorem~\ref{theorem: well-posedness} presents that the unique solution \(\tX\) of the Graphon-NDE is uniformly bounded, which immediately implies that \(\tX\) is square integrable, i.e., \(\tX \in C\left([0, T]; L^{2}(I; \mathbb{R}^{1 \times F})\right)\). Our forthcoming convergence results and rate estimates for GNDE solutions will be formulated in this \(L^2\)-based function space.

\begin{table}[h]
\centering
\label{tab:notation_comparison}
\renewcommand{\arraystretch}{1.25}
\begin{tabular}{l|l|l}
\toprule
\textbf{Object} & \textbf{GNDE (finite graph)} & \textbf{Graphon-NDE (continuum limit)} \\
\midrule
Neural operator 
& $\mathrm{GNN}(\ \cdot\ ;\mS,\tH(t))$ 
& $\mathrm{WNN}(\ \cdot\ ;\tW,\tH(t))$ \\
Dynamics
& $\begin{aligned}
    \frac{d}{dt}\mX(t) &= \mathrm{GNN}(\mX(t);\mS,\tH(t)) \\
    \mX(0) &= \mZ
\end{aligned}$ 
& $\begin{aligned}
    \frac{\partial}{\partial t}\tX(u,t) &= \mathrm{WNN}(\tX(u,t);\tW,\tH(t)) \\
    \tX(u,0) &= \tZ(u)
\end{aligned}$ \\

Graph structure 
& $\begin{aligned}
    &\text{Adjacency matrix $\mW \in [0,1]^{n\times n}$}\\
    &\newrevised{\text{Graph shift operator $\mS=\mW/n$}} 
\end{aligned}$
& Graphon $\tW : [0,1]^2 \to [0,1]$ \\

State representation
& $\mX : [0,T] \to \mathbb{R}^{n\times F}$ 
& $\tX : [0,1]\times[0,T] \to \mathbb{R}^{1\times F}$ \\

Initial condition
& Feature matrix $\mZ \in \mathbb{R}^{n\times F}$ 
& Feature function $\tZ : [0,1] \to \mathbb{R}^{1\times F}$ \\

\bottomrule
\end{tabular}
\caption{GNDEs on finite graphs and their infinite-node limits, Graphon-NDEs. 
Here $\tH(t) \in \mathbb{R}^{L\times F \times F \times K}$ denotes the (time-varying) trainable filter coefficients shared across graph sizes.}
\end{table}

\subsection{Trajectory-Wise Convergence}\label{WNN}

We proceed to study the convergence of GNDEs to Graphon-NDEs in terms of their solution trajectories. Let $\{\mathcal{G}_n\}_{n=1}^\infty$ be a sequence of graphs with adjacency matrices $\{\mW_{\mathcal{G}_n}\}_{n=1}^\infty$. Let the GSO $\mS_{\mathcal{G}_n}$ be defined as the adjacency matrix $\mW_{\mathcal{G}_n}$ normalized by $1/|V(\mathcal{G}_n)|$, i.e., \newrevised{$$\mS_{\mathcal{G}_n} := \mW_{\mathcal{G}_n}/|V(\mathcal{G}_n)|.$$} Recalling \eqref{GNDE}, we formulate a sequence of GNDEs as
\begin{align}\label{sequential GNDEs}
    \begin{split}
        \frac{d}{dt} \mX_{\mathcal{G}_n}(t) &= \GNN(\mX_{\mathcal{G}_n}(t);\mS_{\mathcal{G}_n}, \tH(t)),\\
        \mX_{\mathcal{G}_n}(0) &= \mZ_{\mathcal{G}_n}\in \mathbb{R}^{|V(\mathcal{G}_n)|\times F},
    \end{split}
\end{align}
where $\mZ_{\mathcal{G}_n}$ is the initial node feature matrix for graph $\mathcal{G}_n$. Below we establish the \textit{trajectory-wise} convergence of GNDE solutions to Graphon-NDE solutions.

\begin{theorem}[Trajectory-wise convergence, proof in Appendix \ref{appendix: Proof of convergence of Graphon-NDEs}]\label{theorem: convergence of solutions}
Suppose that {\color{black}\hyperref[AS0]{AS0} and \hyperref[AS1]{AS1}} hold, and let \(\tW \in L^{\infty}(I^2)\) and \(\tZ \in L^{\infty}(I; \mathbb{R}^{1 \times F})\). Let $\tX$ and $\mX_{\mathcal{G}_n}$ denote the solutions of Graphon-NDE \eqref{Graphon-NDE} and GNDE \eqref{sequential GNDEs}, respectively. If $\{(\mathcal{G}_n,\mZ_{\mathcal{G}_n})\}_{n=1}^\infty$ converges to $(\tW,\tZ)$ (cf. Definition \ref{def: graph node feature converges -- formal one}), then for any $T>0$, there exists a sequence $\{\pi_n\}_{n=1}^\infty$ of permutations such that $$\lim_{n\to\infty}\|\tX-\tX_{\pi_n(\mathcal{G}_n)}\|_{C([0,T];L^2(I;\mathbb{R}^{1\times F}))}=0,$$ 
where $\tX_{\pi_n(\mathcal{G}_n)}$ denotes the induced graphon feature function of $\mX_{\pi_n(\mathcal{G}_n)}$.
\end{theorem}

\paragraph{Discussion} 
The norm in the function space \( C([0,T]; L^2(I; \mathbb{R}^{1\times F})) \) involves taking the supremum over \( t \in [0,T] \). Consequently, the convergence we establish is uniform in time; that is, as \( n \to \infty \), the approximation error diminishes uniformly along the entire trajectory, which consists of infinitely many intermediate states. In contrast, the convergence results in the literature for GNNs with finitely many layers~\citep{ruiz2020graphon,keriven2020convergence,maskey2023transferability} establish convergence only at the discrete set of layer outputs as the graph size grows. The trajectory-wise convergence we prove for GNDEs is therefore \revised{strictly} stronger. Moreover, we remark that the established trajectory-wise convergence relies on Gr\"onwall-type inequalities from dynamical systems and stability theory, \revised{which are tools typically not required in the analysis of discrete GNNs.}

Technically, a key challenge in the convergence analysis is the dimensional mismatch between the matrix-valued output of GNDEs and the function-valued output of Graphon-NDEs, making direct comparison infeasible. This is resolved by reformulating GNDEs as equivalent Graphon-NDEs using the induced (piecewise constant) graphon representation, which enables both outputs to be compared within the same underlying function space. 

\newrevised{Theorem \ref{theorem: convergence of solutions} also establishes a stability property of GNDEs: their behavior remains well controlled as the number of nodes increases along structurally similar graphs.} This hinges on the temporal continuity of convolutional filters and Lipschitz conditions for the activation function. The latter assumption aligns with recent empirical studies of GNNs \citep{dasoulas2021lipschitz, arghal2022robust}, which demonstrate that enhanced Lipschitz continuity in GNNs improves robustness, generalization, and performance on large-scale tasks. \revised{Moreover, Theorem \ref{theorem: convergence of solutions} identifies \( C([0,T]; L^2(I; \mathbb{R}^{1\times F})) \) as the function space in which GNDE solutions converge in the continuum regime.} This complements recent advancements in the study of GNN limits and their expressive capabilities \citep{keriven2021universality, keriven2023functions}. 

\subsection{Convergence Rates}

In this section, we use graphons as generative models to construct convergent graph sequences: weighted graphs sampled from smooth graphons and unweighted graphs sampled from $\{0,1\}$-valued (discontinuous) graphons. We further refine our convergence theorem by deriving explicit convergence rates for each case.

\subsubsection{Weighted Graphs}\label{section: weighted graphs}
Let $\tW:{I^2}\rightarrow I$ be a graphon and $\tZ\in L^{\infty}(I;\mathbb{R}^{1\times F})$ be a graphon feature function. For each $n\in\mathbb{N}$, we partition the unit interval \( I \) into \( n \) sub-intervals by defining \( u_i := (i-1)/n \) and \( I_i := [u_i, u_{i+1}) \) for \( i \in [n]\). We construct a sequence of weighted graphs $\{\mathcal{G}_n\}_{n=1}^\infty$, where each graph is defined as $\mathcal{G}_n := \langle [n], [n] \times [n],\mW_{\mathcal{G}_n} \rangle,$
with adjacency matrix $\mW_{\mathcal{G}_n}\in\mathbb{R}^{n\times n}$ generated by direct sampling on the graphon $\tW$ over the mesh grid as
\begin{equation}\label{coeff W for weighted graphs}
[{\mW}_{\mathcal{G}_n}]_{ij} := \tW(u_i,u_j),\quad i,j\in [n].
\end{equation}
  The corresponding node feature matrix $\mZ_{\mathcal{G}_n}\in\mathbb{R}^{n\times F}$ of graph $\mathcal{G}_n$ is generated by sampling on the graphon feature function $\tZ$ as
\begin{equation}\label{coeff G for weighted graphs}
    [\mZ_{\mathcal{G}_n}]_{i,:} := \tZ(u_i), \quad i \in [n].
\end{equation}

This weighted graph model is particularly well-suited for applications requiring fully connected network structures, such as dense communication networks and recommendation systems \citep{barrat2004architecture,newman2004analysis,aggarwal2016recommender}. In these settings, the graphons are typically assumed to be \revised{H\"older} continuous, reflecting the fact that interactions between entities (e.g., users, devices, or items) evolve gradually and predictably. We summarize the assumptions below. 

\begin{itemize}
    \item\label{AS2} \textbf{AS2.} The graphon \( \tW \) is \( (A_1,\revised{\alpha_1}) \)-\revised{H\"older continuous for $A_1>0$ and $\alpha_1\in(0,1]$}, that is, \(|\tW(u_2, v_2) - \tW(u_1, v_1)| \leq A_1(|u_2 - u_1| + |v_2 - v_1|)^\revised{\alpha_1}\), for all $v_1,v_2,u_1,u_2\in I$.

    \item\label{AS3} \textbf{AS3.} The initial graphon feature function \( \tZ=[Z_f:f\in[F]]\in L^{\infty}(I;\mathbb{R}^{1\times F}) \) is \revised{$(A_2,\alpha_2)$-H\"older continuous for $A_2>0$ and $\alpha_2\in(0,1]$}, that is, for each $f\in\mathbb[F]$, $|Z_f(u_2)-Z_f(u_1)|\leq A_2|u_2-u_1|^{\revised{\alpha_2}}$, for all $u_1,u_2\in I$.

\end{itemize}
\begin{theorem}[Rates for weighted graphs, proof in Appendix \ref{Appendix: proof of convergence rates}]\label{theorem: rate of Lipschitz}
Suppose that {\color{black}\hyperref[AS0]{AS0}-\hyperref[AS3]{AS3}} hold. Let the adjacency matrices and node feature matrices of graphs $\{\mathcal{G}_n\}_{n=1}^\infty$ be generated according to \eqref{coeff W for weighted graphs} and \eqref{coeff G for weighted graphs}, respectively. Let $T\in\mathbb{R}^+$ and \revised{$\alpha:=\min\{\alpha_1,\alpha_2\}$}. Let $\tX$ be the solution of Graphon-NDE \eqref{Graphon-NDE} and $\tX_{\mathcal{G}_n}$ be the induced graphon function of the solution $\mX_{\mathcal{G}_n}$ of GNDE \eqref{sequential GNDEs}. Then it holds that 
\begin{equation}\label{eq: rate of weighted graph}
    \|\tX-\tX_{\mathcal{G}_n}\|_{C([0,T];L^2(I;\mathbb{R}^{1\times F}))}\leq \frac{C}{n^{\alpha}},
\end{equation}
where $C$ is constant independent of $n$ and with explicit formula provided in equation \eqref{constant C in first convergence rate}. As a result, for any $n_1,n_2\in\mathbb{N}$, it holds that
\begin{equation}\label{weighted graph: size transferability bounds}
\|\tX_{\mathcal{G}_{n_1}}-\tX_{\mathcal{G}_{n_2}}\|_{C([0,T];L^2(I;\mathbb{R}^{1\times F}))}\leq C\left(\frac{1}{n_1^{\alpha}}+\frac{1}{n_2^{\alpha}}\right).
\end{equation}

\end{theorem}

\paragraph{Discussion} \revised{Theorem~\ref{theorem: rate of Lipschitz} establishes an $\mathcal{O}(1/n^\alpha)$ convergence rate for weighted graphs and initial features sampled from their respective underlying H\"older continuous functions.} This rate is known to be optimal for approximating H\"older continuous functions \revised{by piecewise constant functions}~\citep{schumaker2007spline}. Furthermore, the rate for GNDEs we obtain is trajectory-wise (i.e., uniform-in-time), which is strictly stronger than the linear convergence rates established for discrete-layer GNNs on Lipschitz continuous graphons $(\revised{\alpha_1}=1)$ ~\citep{maskey2023transferability,krishnagopal2023graph}. %

\subsubsection{Unweighted Graphs}\label{section: unweighted graphs}
Let $\tW: {I^2}\rightarrow \{0,1\}$ be a binary-valued graphon and $\tZ\in L^{\infty}(I;\mathbb{R}^{1\times F})$ be a graphon feature function. We denote by $\tW^+$ the support set of function $\tW$, that is $\tW^+:=\{(u,v):\tW(u, v) = 1\}$. We construct a sequence of unweighted graphs $\{\mathcal{G}_n\}_{n=1}^\infty$ where $\mathcal{G}_n := \langle [n], E(\mathcal{G}_n),\mW_{\mathcal{G}_n} \rangle$, with edge set $E(\mathcal{G}_n)$ defined by $E(\mathcal{G}_n) := \{(i, j) \in [n]\times[n]: (I_i \times I_j) \cap \tW^+ \neq \emptyset \}$, and adjacency matrix $\mW_{\mathcal{G}_n} $ defined as
\begin{equation}\label{coeff W for simple graphs}
    [ \mW_{\mathcal{G}_n} ]_{ij} := \begin{cases}
        1,& \text{if } (i,j) \in E(\mathcal{G}_n),\\
        0, & \text{otherwise},
    \end{cases}
\end{equation} 
The corresponding node feature matrix $\mZ_{\mathcal{G}_n}$ for graph $\mathcal{G}_n$ is generated, from a Lipschitz continuous graphon feature function $\tZ$, as
\begin{equation}\label{coeff G for simple graphs}
[\mZ_{\mathcal{G}_n}]_{i,:} := \frac{1}{|I_i|}\int_{I_i} \tZ(u) \, du, \quad i \in [n].
\end{equation}
This model is  for generating  network structures with binary relations, which are prevalent in social networks, citation graphs, and biological networks \citep{jeong2000large,milo2002network,girvan2002community,leskovec2009community,easley2010networks}. 

The discontinuity of graphons prevents \hyperref[AS2]{AS2} from being satisfied. To tackle this issue, we \revised{adopt} the upper box-counting dimension \citep{falconer2014fractal} for the boundary $\partial\tW^+$, where $\tW^+$ is the support of the graphon $\tW$. \revised{The application of the box-counting dimension to graph limits was established by \citet{medvedev2014nonlinear} for the nonlinear heat equation on dense graphs.} We review the definition of upper box-counting dimension as follows. Let $\Omega$ be any non-empty bounded subset of $\mathbb{R}^2$ and let $\mathcal{N}_\delta(\Omega)$ be the number of $\delta$-mesh cubes that intersect $\Omega$. The upper box-counting dimensions of $\Omega$ is defined as 
\begin{equation}\label{definition of dim box counting}
\overline{\operatorname{dim}}_{\mathrm{B}} \Omega:=\varlimsup_{\delta \rightarrow 0} \frac{\log \mathcal{N}_\delta(\Omega)}{-\log \delta}.
\end{equation}
It is clear that 
$\overline{\operatorname{dim}}_{\mathrm{B}}(\Omega)\in [0,2]$ for any non-empty bounded subset $\Omega$ of $\mathbb{R}^2$. As a simple example, the straight line $\{(x,0):x\in [0,1]\}$ has an upper box-counting dimension of $1$. 

\begin{theorem}[Rates for unweighted graphs, proof in Appendix \ref{Appendix: proof of convergence rates}]\label{theorem: rate of simple graph}
Suppose that {\color{black}\hyperref[AS0]{AS0}, \hyperref[AS1]{AS1} and \hyperref[AS3]{AS3}} hold. Let $\tW:{I^2}\to\{0,1\}$ be a graphon for unweighted graphs with $b:=\overline{\mathrm{dim}}_{\mathrm{B}}(\partial \tW^+)\in[1,2)$. Let the adjacency matrices and node feature matrices of graphs $\{\mathcal{G}_n\}_{n=1}^\infty$ be generated according to \eqref{coeff W for simple graphs} and \eqref{coeff G for simple graphs}, respectively. Let $T\in\mathbb{R}^+$. Let $\tX$ be the solution of Graphon-NDE \eqref{Graphon-NDE} and $\tX_{\mathcal{G}_n}$ be the induced graphon function of the solution $\mX_{\mathcal{G}_n}$ of GNDE \eqref{sequential GNDEs}. Then for any $\epsilon\in(0,2-b)$, there exists a positive integer $N_{\epsilon,\tW}$ (depending on $\epsilon$ and $\tW$) such that when $n>N_{\epsilon,\tW}$, it holds that 
\begin{equation}\label{eq: rate of unweighted graph}
    \|\tX-\tX_{\mathcal{G}_n}\|_{C([0,T];L^2(I;\mathbb{R}^{1\times F}))}\leq \frac{\widetilde C}{n^{\revised{\min\{1-\frac{b+\epsilon}{2},\alpha_2\}}}},
\end{equation}
where $\widetilde C$ is a constant independent of $n$, and with explicit formula provided in equation \eqref{tilde C in rate of simple graphs}. As a result, for any $n_1,n_2>N_{\epsilon,\tW}$, it holds that
\begin{equation}\label{unweighted graph: size transferability bounds}
\|\tX_{\mathcal{G}_{n_1}}-\tX_{\mathcal{G}_{n_2}}\|_{C([0,T];L^2(I;\mathbb{R}^{1\times F}))}\leq \widetilde C\left(\frac{1}{n_1^{\revised{\min\{1-\frac{b+\epsilon}{2},\alpha_2\}}}}+\frac{1}{n_2^{\revised{\min\{1-\frac{b+\epsilon}{2},\alpha_2\}}}}\right). 
\end{equation} 

\end{theorem}

\paragraph{Discussion}The $\epsilon>0$ in Theorem \ref{theorem: rate of simple graph} is a pre-specified parameter that can be chosen arbitrarily small, making the convergence rate in Theorem \ref{theorem: rate of simple graph} \emph{almost} $\mathcal{O}\left(\revised{1/n^{\min\{1 - b/2,\alpha_2\}}}\right)$. \newrevised{Here, the Graphon-NDE depends on the equivalence class of $\tW$, whereas the convergence rate is tied to the specific representative used to generate the graphs $\{\mathcal{G}_n\}$. Accordingly, the box-counting dimension $b$ reflects the geometric regularity of the chosen representative, which may change under modifications on null sets or measure-preserving rearrangements, rather than being determined solely by the equivalence class.} \revised{To simplify the following discussion, we specifically focus on the case where the initial feature function is Lipschitz continuous ($\alpha_2=1$ in \hyperref[AS3]{AS3}). In this setting, the convergence rate is dominated by the graphon approximation error, yielding a rate of approximately $\mathcal{O}\parens{1/n^{1-b/2}}$.} In contrast to the rate for weighted graphs established in Theorem \ref{theorem: rate of Lipschitz}, the rate for unweighted graphs relies on the complexity of the boundary $\partial \tW^+$, measured by its upper box-counting dimension $b$. The more intricate \( \partial \tW^+ \) is, leading to the larger value of \( b \), the poorer the convergence rate becomes. For boundaries with box-counting dimension \( b = 1 \) (e.g., smooth curves or piecewise linear segments), convergence is relatively fast at rate \( \mathcal{O}(1/n^{0.5}) \). For boundaries with greater fractal complexity, where \( b \in (1, 2) \) (e.g., moderately irregular or self-similar structures such as the hexaflake), convergence slows to \( \mathcal{O}(1/n^c) \) for some \( c \in (0, 0.5) \). We note that numerical experiments (see hierarchical stochastic block model (HSBM) and hexaflake graphons in Figure~\ref{fig:multi_graph_experiment}) suggest that our theoretical rate for unweighted graphs may be \emph{pessimistic}, reflecting a worst-case scenario. Empirically, faster convergence rates are observed. In addition, we find that the HSBM graphon appears to yield faster convergence than the hexaflake, likely due to its smaller box-counting dimension. This observation is consistent with the trend indicated in Theorem~\ref{theorem: rate of simple graph}, where a larger box-counting dimension corresponds to a slower convergence rate.

We mention that the graphons for unweighted graphs are discontinuous and prior studies on GNNs \citep{ruiz2021graphon,ruiz2021graphon1,morency2021graphon,maskey2023transferability} lack convergence rates for this case. In contrast, our result goes beyond GNNs and establishes trajectory-wise rates for GNDEs over unweighted graphs, using a novel analysis based on the box-counting dimension.

\revised{We finally remark that the convergence established in Theorem \ref{theorem: convergence of solutions} relies crucially on the existence of an admissible sequence of permutations to be valid. In contrast, Theorems \ref{theorem: rate of Lipschitz} and \ref{theorem: rate of simple graph} do not require permutations; the specific way of constructing the graphs in these theorems is sufficient to guarantee convergence without reordering.}

\section{Numerical Experiments}

\subsection{Graphon Convergence Rates}\label{sec:graphon_conv_rate_experiment}

\paragraph{Graphons} 

To empirically verify Theorem~\ref{theorem: rate of Lipschitz},  we examine the convergence behavior of the tent graphon \citep{xia2023implicitgraphonneuralrepresentation}, a weighted smooth graphon defined by $\tW(u,v) = 1 - |u-v|^{\alpha}$, $u,v\in I$, with $\alpha=\tfrac{1}{2}$ (Hölder-$\tfrac{1}{2}$) or $\alpha=1$ (Lipschitz).

For verification of Theorem~\ref{theorem: rate of simple graph}, two $\{0,1\}$-valued graphons with varying box-counting dimension are considered. We examine the HSBM graphons \citep{holland1983stochastic,crane2015framework} with multiscale community structure, where the box-counting dimension of the support is 1 with a controllable grid size parameter. We also consider the hexaflake fractal, a Sierpiński \(n\)-gon–based construction that has been used in certain practical design applications \citep{hexaflake_cite}, as a \revised{novel} graphon with box-counting dimension of $\frac{\log(7)}{\log(3)}$ or about $1.77$.

\begin{figure}[H]
    \centering
    \includegraphics[width=\textwidth]{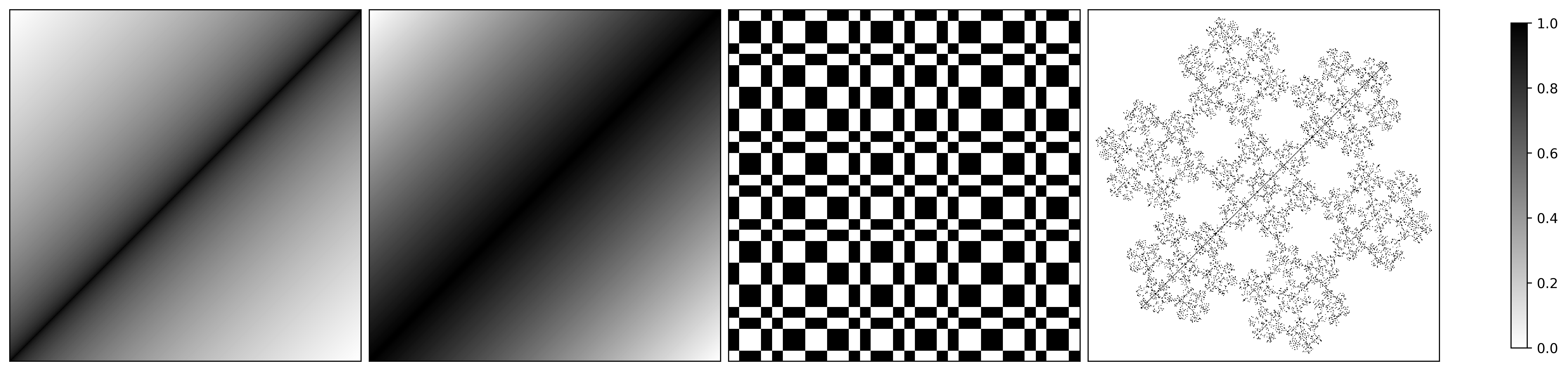}
    \caption{Hölder Tent (left), Tent (center-left), HSBM (center-right), and Hexaflake (right) graphon visualizations.}
    \label{fig:three_graphon}
\end{figure}

\paragraph{Experiment setup}
We use GNDEs parameterized with a two-layer GNN, based on the models of \citet{poli2019graph}, where both the feature and hidden dimensions are \(1\) \newrevised{($F = 1$) so that each layer inputs and outputs a single feature. GNDEs} share the same \newrevised{\textit{constant-in-time filters}} with entries bounded in \([-1, 1]\). \revised{Filters are initialized randomly and are not trained.}  The initial conditions are random Fourier polynomials of degree $D$, defined by $\tZ(u) := \sum_{k=1}^{D} a_k \cos(2 \pi k u) + b_k \sin(2 \pi k u)$,
where $a_k$ and $b_k$ are chosen randomly, creating diverse and smooth signals over graph nodes. The details are in  Appendix \ref{appendix: additionalgraphons}.

\paragraph{Evaluation}  
To approximate the graphon solution $\tX$, we use a reference graph with $N_{\mathrm{largest}} = 5000$ nodes. We present the log-log convergence plot of $\max_{t}\frac{\|\mX_n(t)-\mX_{5000}(t)\|_2}{\|\mX_{5000}(t)\|_2}$ for number of nodes $n$ ranging from $550$ to $1950$ with a step size of $100$. This approximates the maximal relative error over all $t \in [0,1]$ of GNDE evolution. We evolve GNDEs through the Dormand-Prince method of order $5$ \citep{DORMAND198019}. We plot the resulting curves in Figure \ref{fig:multi_graph_experiment}.

\begin{figure}[h!]
    \centering
    \begin{minipage}{0.48\textwidth}
        \includegraphics[width=\linewidth]{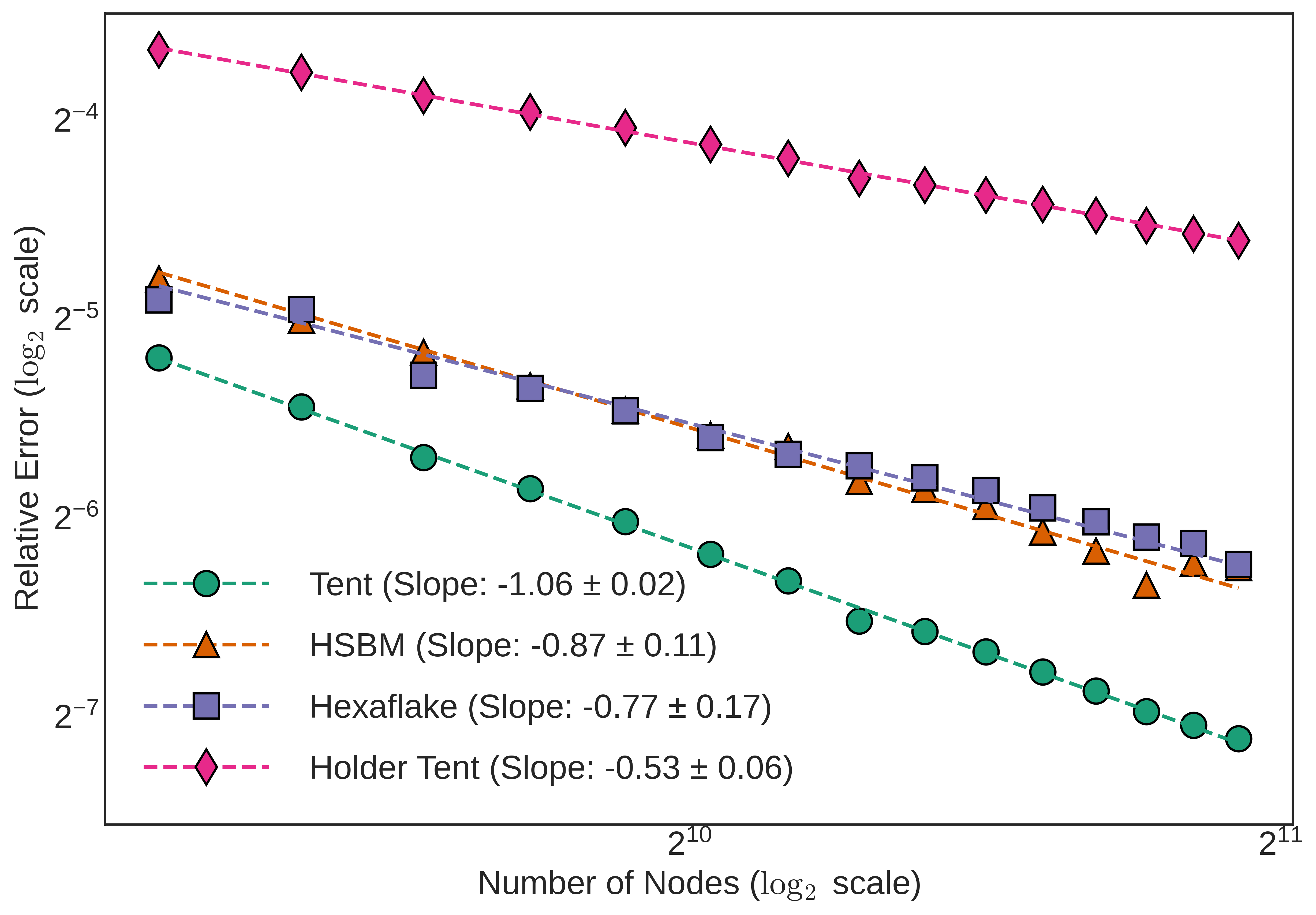}
        \vspace{0.5em}
    \end{minipage}
    \hspace{1em} %
    \begin{minipage}{0.48\textwidth}
        \vspace{-0.5em}
        \caption{\textbf{Convergence rates of GNDE solutions.} 
        Mean relative errors between GNDE and Graphon-NDE solutions on graphs sampled from four graphons: (1) \emph{Tent graphon} (Lipschitz), matching  \(\mathcal{O}(1/n)\) rate in  Theorem~\ref{theorem: rate of Lipschitz}, (2) \emph{HSBM graphon} (box counting dimension 1), (3) \emph{Hexaflake graphon} (fractal boundary with box counting dimension 1.77), and (4) \emph{Hölder Tent} which is Hölder-$\tfrac{1}{2}$ and exhibits a rate near \(\mathcal{O}(1/n^{0.5})\) as expected from Theorem~\ref{theorem: rate of Lipschitz}. The HSBM graphon yields faster convergence than the hexaflake, consistent with the trend indicated in Theorem~\ref{theorem: rate of simple graph}. We refer to Figure \ref{fig:error_bars_graphon_main} in Section \ref{appendix: additionalgraphons} for their convergence plots with error bars.}
        \label{fig:multi_graph_experiment}
    \end{minipage}
\end{figure}

\paragraph{Adddtional graphons} We also include three additional cases: one Lipschitz graphon with a larger Lipschitz constant, and two binary graphons: a checkerboard graphon (with box-counting dimension~1) and a Sierpiński graphon (with box-counting dimension~1.89) as presented in Section~\ref{appendix: additionalgraphons}, with their convergence plots shown in Figure~\ref{fig:error_bars_graphon_additional}. The observed phenomenon is consistent with the results in Figure~\ref{fig:multi_graph_experiment}.

\paragraph{Discussion}  
As shown in Figure~\ref{fig:multi_graph_experiment}, for the considered feature functions the empirical convergence rates are consistently better than the theoretical rate predicted by Theorem~\ref{theorem: rate of simple graph}, suggesting that the theorem provides a \emph{pessimistic} bound. In practice, one often observes faster convergence. This gap arises because our theoretical rate characterizes the worst-case scenario, whereas in many applications the signals of interest lie in low-dimensional or low-frequency subspaces, thereby avoiding regions where convergence is intrinsically slow. Furthermore, we observe a clear dependence of the convergence rate on the box-counting dimension: the more complex the boundary, the slower the convergence. This highlights a fundamental link between graphon boundary complexity and the scalability of graph neural differential equations.

\subsection{Real Data Node Classification}\label{appendix: real data experiments}

\revised{In this section, we explore the size transferability of GNDEs on real node classification tasks using  benchmark real-world datasets by first training a model on a subgraph and then assessing its performance on the full graph.
For node classification, the GNDE solution at the terminal time provides node embeddings, which are then mapped to class labels through a fixed readout layer.
Accordingly, changes in classification accuracy under size transfer reflect changes in the GNDE solution induced by perturbations of the underlying graph, which we quantify using induced graphon representations constructed with respect to the original node indexing.}

\paragraph{Graph Datasets} 
 We adopt a variety of widely used benchmark node classification datasets. We examine the homophilic citation networks Cora \citep{mccallum2000cora}, Pubmed \citep{namata2012pubmed} and  Citeseer \citep{sen2008citeseer} \revised{where similar nodes are connected}, and we adopt \revised{standard} fixed \revised{(Planetoid)} splits \citep{yang2016revisiting}. We also consider the heterophilic graph datasets Cornell, Texas, Wisconsin \citep{webkb1998}, the Squirrel and Chameleon datasets \citep{rozemberczki2021multiscaleattributednodeembedding}, and the Actor dataset \citep{tang2009actor}, \revised{where dissimilar nodes are connected and learning may be more difficult,} each with randomized $60/20/20$ splits. For a large scale example, we consider the ogbn-arxiv dataset \citep{hu2021opengraphbenchmarkdatasets} using standard splits. Their statistics are summarized in Table~\ref{table:dataset_structure} in Appendix~\ref{appendix: real data details}.

\paragraph{Graph construction}  
Each dataset consists of a graph with binary edge values ($\{0,1\}$) and associated node features. From the full graph, we extract sequences of random subgraphs of varying sizes, retaining between $10\%$ and $90\%$ of all nodes. For each proportion, we independently sample the specified percentage of nodes from the training, validation, and test sets. The selected nodes from all three sets are then combined to form a subgraph. This sampling process ensures that the resulting subgraphs maintain a balanced representation of nodes across all classes.

\paragraph{Experiment setup} We train GNDE models on each random subgraph for each dataset. For each node classification task on a subgraph, we create a corresponding GNDE model which consists of three layers: a GNN head ($L=1,K=2$) mapping the high dimensional initial input features to lower dimensional input features for the GNDE, a GNDE parameterized by a GNN with fixed input, output, and hidden dimensions ($L=2,K=2$), and a linear GNN readout layer ($L=1,K=1$) mapping the output of the dynamics to the class labels for the final classification task. After training, model weights are transferred to the full graph. \revised{In practice, this means that model weights are kept constant after training, corresponding to zero-shot transfer, but the trained model is now applied to the full graph with the inherited graph shift operator and node data.}

\paragraph{Implementation details} We implement GNDEs using the \texttt{torchdiffeq} library \citep{chen2018neural} and the code provided by \citet{poli2019graph}. \revised{Model filters} are trained using cross-entropy loss optimized with Adam \citep{kingma2014adam}. \textit{Our primary objective is not to achieve state-of-the-art performance on the node classification task}, but to analyze the generalization and transfer behavior of GNDEs under a standardized training protocol. Details regarding hyperparameter selection, training, and computational complexity are provided in Appendix~\ref{appendix: real data details}.

\paragraph{Evaluation}
We require meaningful evaluation metrics to assess performance. We gather three test accuracies for each GNDE trained on a subgraph: the test accuracy on the associated subgraph (\textbf{STA}), the test accuracy on the full graph after transfer (\textbf{FTA}), and the test accuracy on the subset of test nodes associated with the subgraph but after transfer (\textbf{SFTA}).

We also define secondary metrics to measure transfer and graphon error. Suppose that $\mathcal{G}$ is the full graph with adjacency matrix $\mW_{\mathcal{G}}$, and $\mathcal{G}'$ is a subgraph of $\mathcal{G}$ with adjacency matrix $\mW_{\mathcal{G}'}$. By $\tW_{\mathcal{G}}$ and $\tW_{\mathcal{G}'}$ we denote the induced graphon representation (cf. equation \eqref{induced graphon Wn}) of $\mW_{\mathcal{G}}$ and $\mW_{\mathcal{G}'}$, respectively. We define
\begin{align}\label{app:graphonerr}
\text{graphon error}:=\frac{\|\tW_{\mathcal{G}'} - \tW_{\mathcal{G}}\|_{L^2(I)}}{\|\tW_{\mathcal{G}}\|_{L^2(I)}}.
\end{align} 

We also define transfer errors on subgraphs (cf. equation \eqref{app:transfererr}), which are calculated as the relative difference of \textbf{STA} and \textbf{SFTA}: 
\begin{align}\label{app:transfererr}
\text{transfer error}:=\frac{|\textbf{STA} - \textbf{SFTA}|}{|\textbf{STA}|}.
\end{align}
Our evaluation methodology aligns with that used for GNNs in prior work \citep{ruiz2020graphon}.

\begin{figure}[!htbp]
\centering
\includegraphics[width=\textwidth]{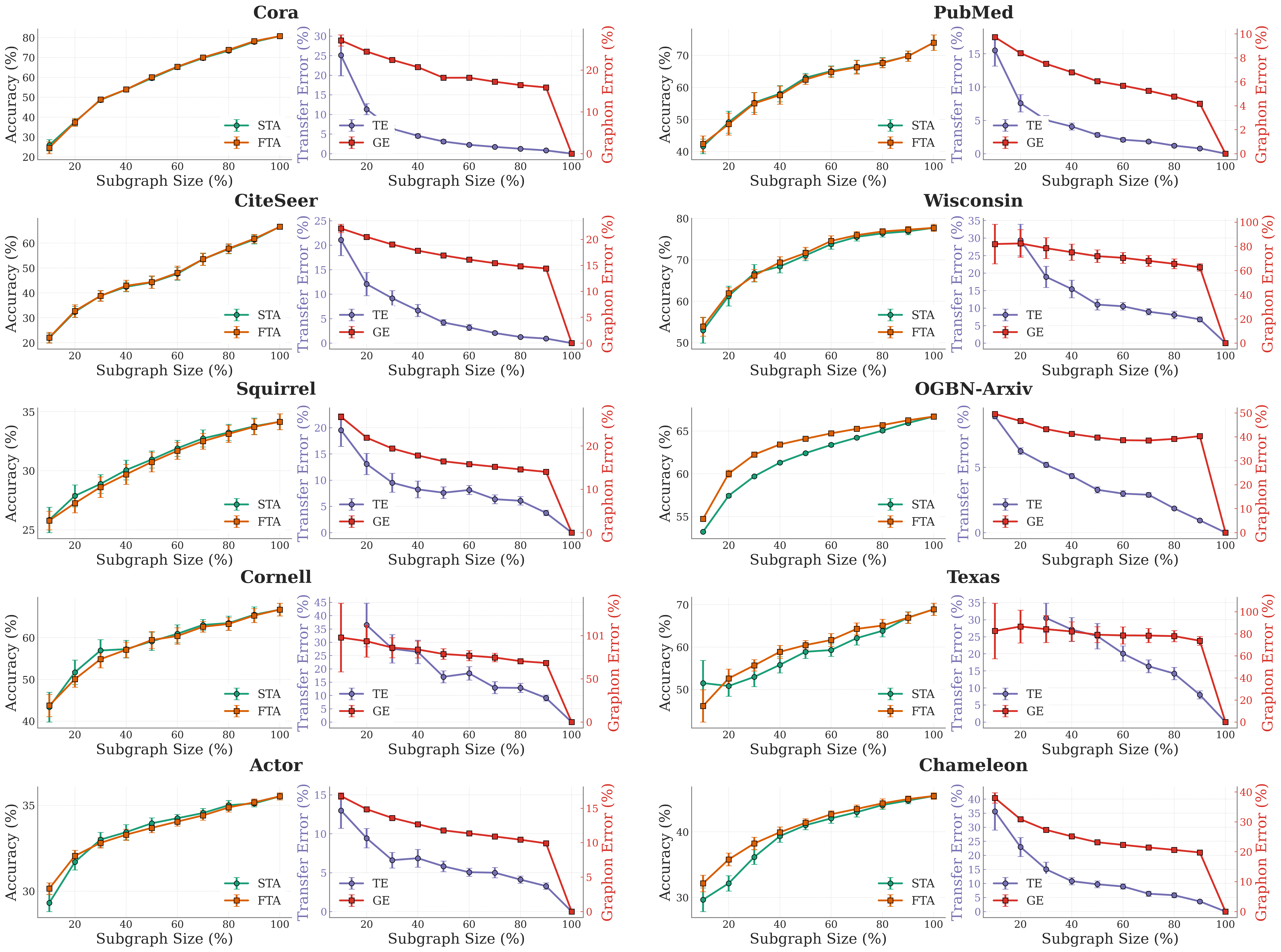}
\caption{Node classification experiment results, with two plots for each dataset. \textbf{Left:} Subgraph test accuracy (\textbf{STA}) and full graph test accuracy (\textbf{FTA}). \textbf{Right:} Transfer error (\textbf{TE}) and graphon error (\textbf{GE}).}
\label{fig:large_grid_convergence}
\end{figure}

\paragraph{Discussion} Numerical results are plotted in Figure \ref{fig:large_grid_convergence} in the form of mean $\pm$ standard deviation. We observe that transfer errors on subgraphs decay as their size increases, a trend consistent with the decay of graphon errors. \revised{This behavior is consistent with the {stability predictions} implied by Theorem~\ref{theorem: convergence of solutions}.} Additionally, we find that $\textbf{STA}$ and $\textbf{FTA}$ are numerically similar and increasing as the size increases, further supporting the generalization capability of GNDEs.

\paragraph{Computational time}

Below we report training and inference times for the models used in each experiment, broken down by the size of the training subgraph. Training times are reported as an average per $200$ epochs, and inference times are reported as the average for inference on the full graph. \newrevised{For smaller datasets, the size of the training subgraph does not meaningfully impact overall training time due to fixed and dominating computational costs such as data loading, GPU initialization, and communication overhead. However}, the larger ogbn-arxiv shows a clear trend. Training on a subgraph containing only 50\% of the nodes achieves 64\% accuracy with an average training time of 8.94 seconds. In comparison, training on the full graph yields 67\% accuracy but requires about 19.98 seconds. This demonstrates the feasibility of achieving comparable accuracy while cutting training time by more than half for large-scale graphs. This result both empirically supports our theoretical claim that GNDEs trained on smaller graphs can effectively generalize to larger ones and provides meaningful motivation for adopting this approach.

\begin{table}[h!]
\centering
\label{tab:training_times}
\begin{tabular}{lcccc}
\toprule
Dataset & 10\% Subgraph & 50\% Subgraph & 100\% Full Graph & Inference Time \\ 
\midrule
actor & $2.38 \pm 0.08$ & $2.62 \pm 0.05$ & $3.03 \pm 0.07$ & $0.0098 \pm 0.0007$  \\ 
chameleon & $2.40 \pm 0.05$ & $2.63 \pm 0.05$ & $2.97 \pm 0.24$ & $0.0098 \pm 0.0010$ \\ 
cornell & $2.39 \pm 0.09$ & $2.41 \pm 0.05$ & $2.32 \pm 0.14$ & $0.0079 \pm 0.0008$ \\ 
citeseer & $2.13 \pm 0.15$ & $2.38 \pm 0.17$ & $2.69 \pm 0.17$ & $0.0107 \pm 0.0008$ \\ 
cora & $2.04 \pm 0.05$ & $2.09 \pm 0.05$ & $2.17 \pm 0.05$ & $0.0076 \pm 0.0005$ \\ 
pubmed & $2.05 \pm 0.08$ & $2.25 \pm 0.04$ & $2.79 \pm 0.09$ & $0.0123 \pm 0.0009$ \\ 
ogbn-arxiv & $2.16 \pm 0.05$ & $8.94 \pm 0.03$ & $19.98 \pm 0.03$ & $0.0650 \pm 0.0002$ \\ 
squirrel & $2.51 \pm 0.12$ & $2.69 \pm 0.11$ & $3.27 \pm 0.10$ & $0.0079 \pm 0.0001$ \\ 
texas & $2.37 \pm 0.11$ & $2.40 \pm 0.06$ & $2.24 \pm 0.17$ & $0.0085 \pm 0.0007$ \\ 
wisconsin & $2.37 \pm 0.28$ & $2.59 \pm 0.48$ & $2.23 \pm 0.25$ & $0.0051 \pm 0.0001$ \\ 
\bottomrule
\end{tabular}
\vspace{6pt}
\caption{Average training time per 200 epochs (seconds), and average inference time (seconds), \revised{which measures the time required to make predictions on the test set using the trained model.}}
\end{table}

\subsection{Fitting graph diffusion dynamics using GNDEs}\label{appendix: Synthetic Dynamics Experiments}

In this section, we present numerical results demonstrating that GNDE models parameterized by spectral GCNs can fit graph-based diffusion dynamics governed by the nonlocal diffusion equations in Table \ref{tab:dynamics} with good accuracy.

\paragraph{Graphons}

For each graphon in Figure \ref{fig:all_graphons}, we construct four discrete graphs with sizes ranging from $50$ to $200$ through the sampling processes detailed in Sections 3.3.1 and 3.3.2. For the unweighted case, we consider the binary HSBM, Koch curve, and hexaflake fractal graphons. For the weighted case, we consider the product graphon defined by $\tW(x,y) = xy$ and the H\"older tent graphon.

\begin{figure}[htbp]
\centering
\begin{subfigure}[c]{0.18\textwidth}
    \centering
    \textbf{HSBM}\\[0.3em]
    \includegraphics[width=\textwidth]{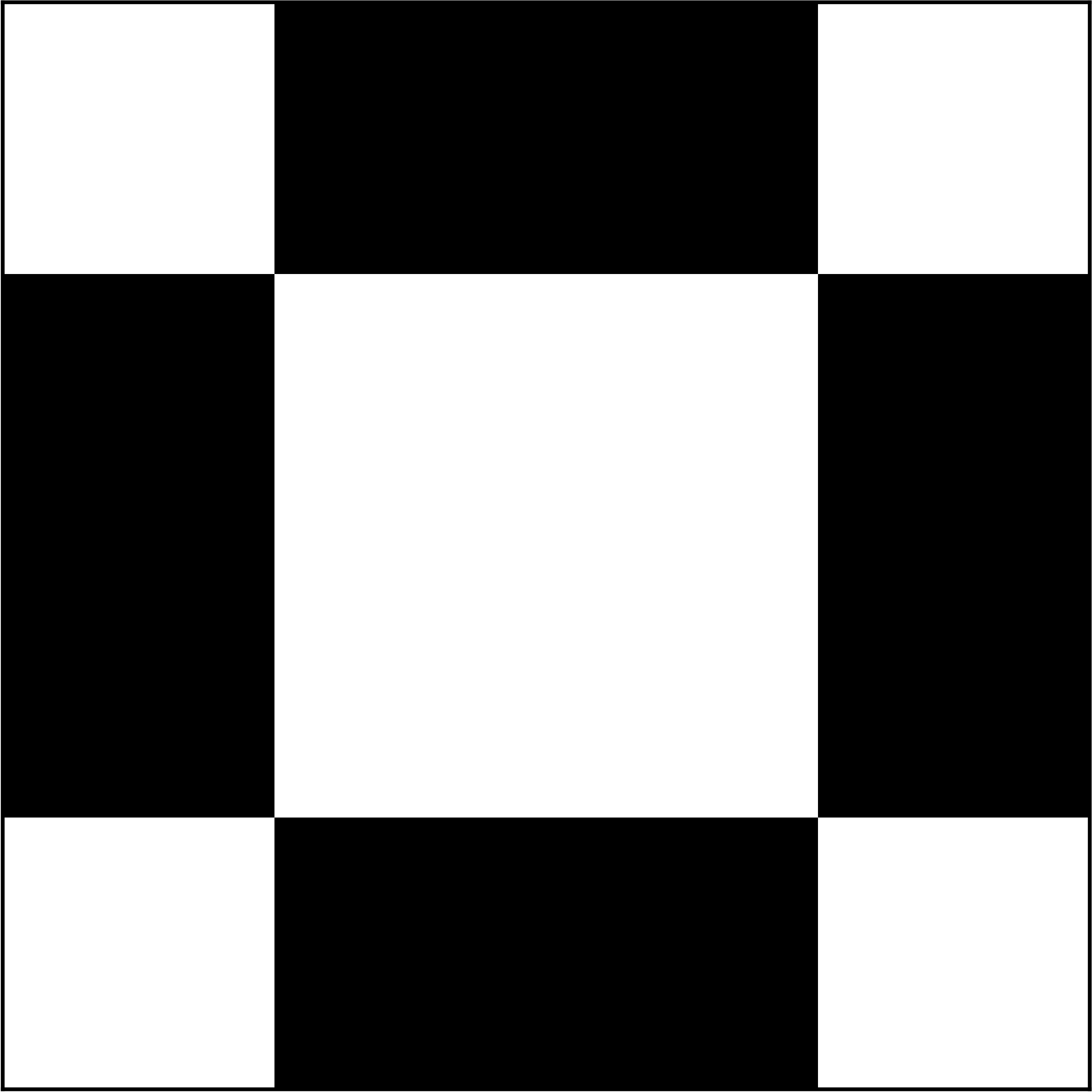}
\end{subfigure}%
\hfill%
\begin{subfigure}[c]{0.18\textwidth}
    \centering
    \textbf{Koch}\\[0.3em]
    \includegraphics[width=\textwidth]{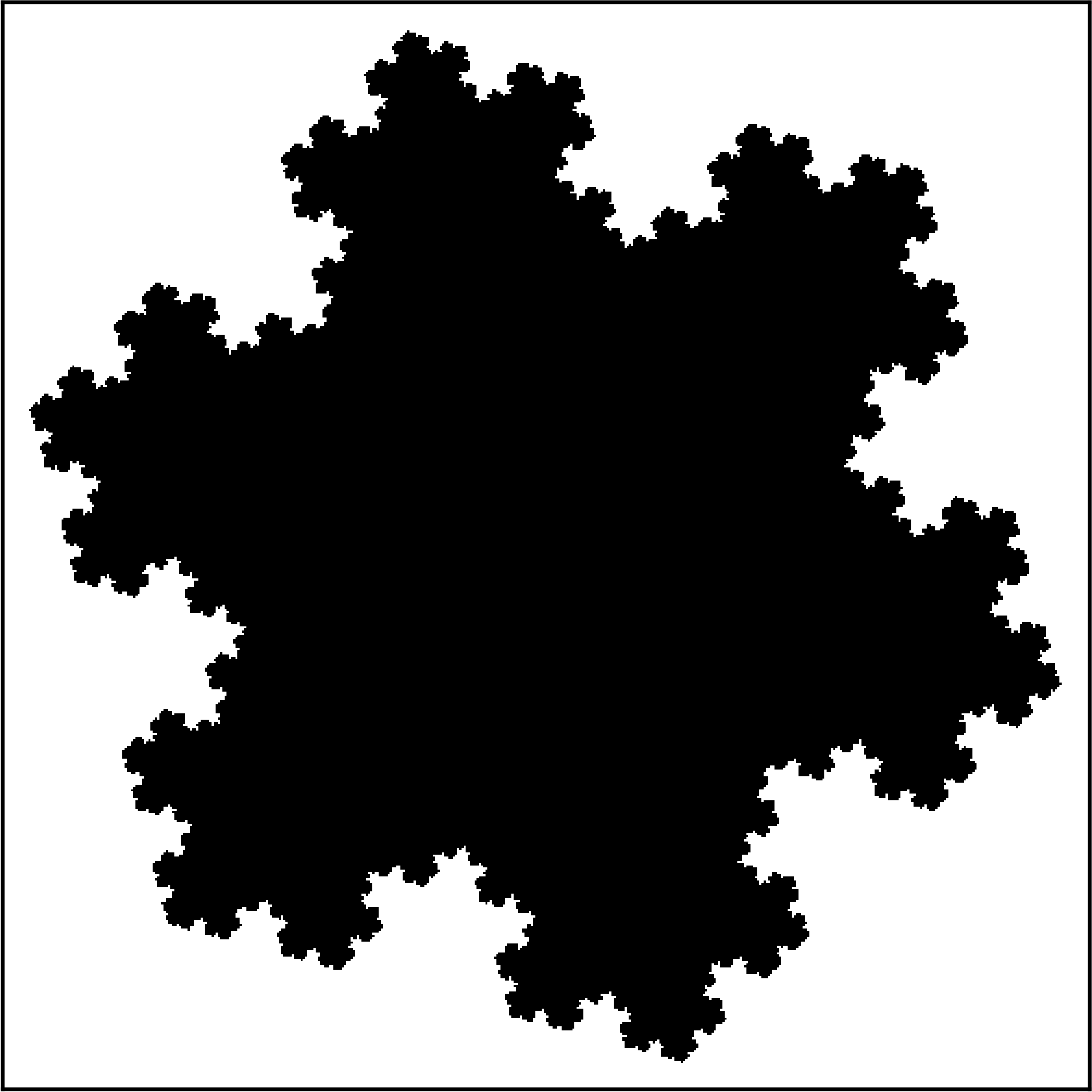}
\end{subfigure}%
\hfill%
\begin{subfigure}[c]{0.18\textwidth}
    \centering
    \textbf{Hexaflake}\\[0.3em]
    \includegraphics[width=\textwidth]{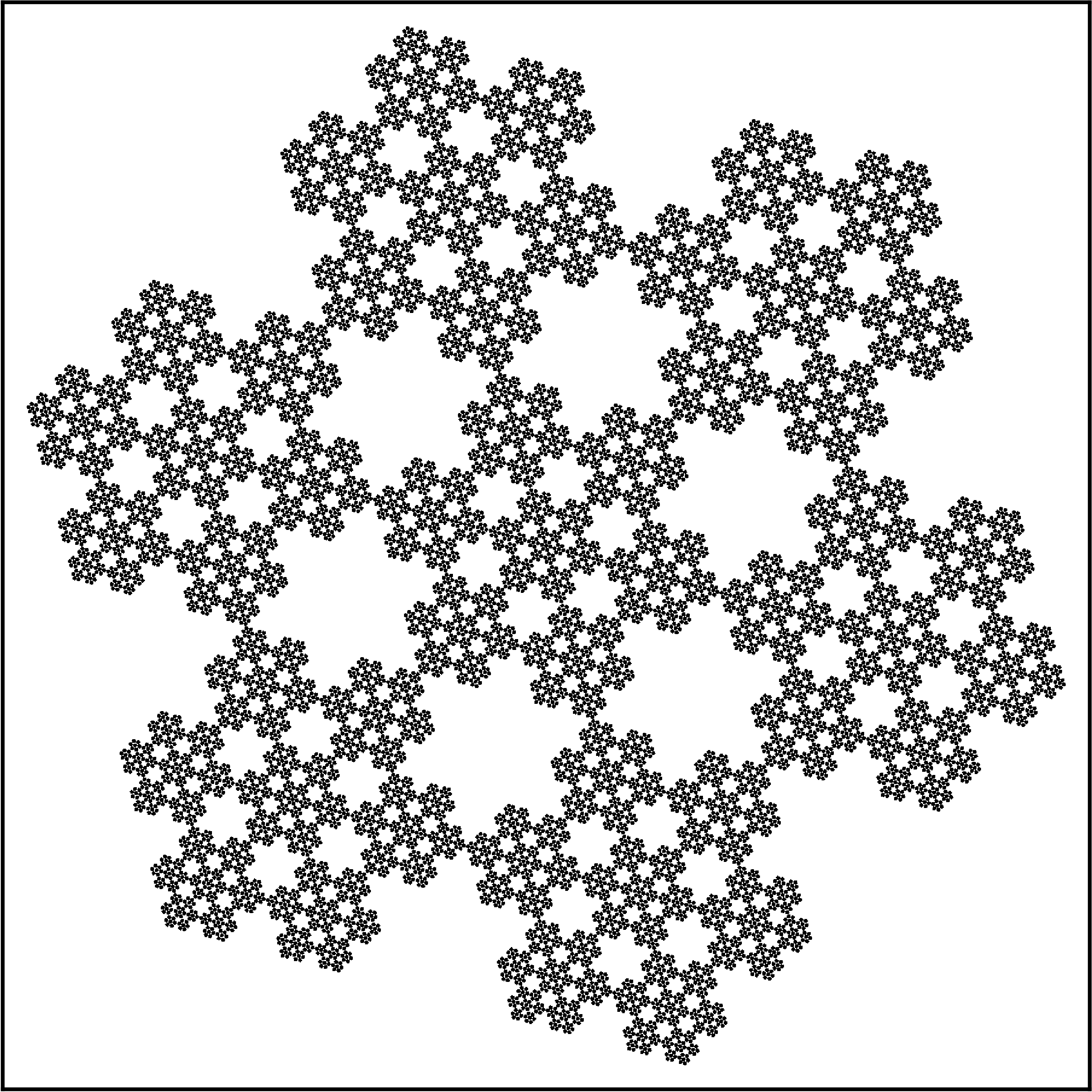}
\end{subfigure}%
\hfill%
\begin{subfigure}[c]{0.2\textwidth}
    \centering
    \textbf{Product}\\[0.3em]
    \includegraphics[width=\textwidth]{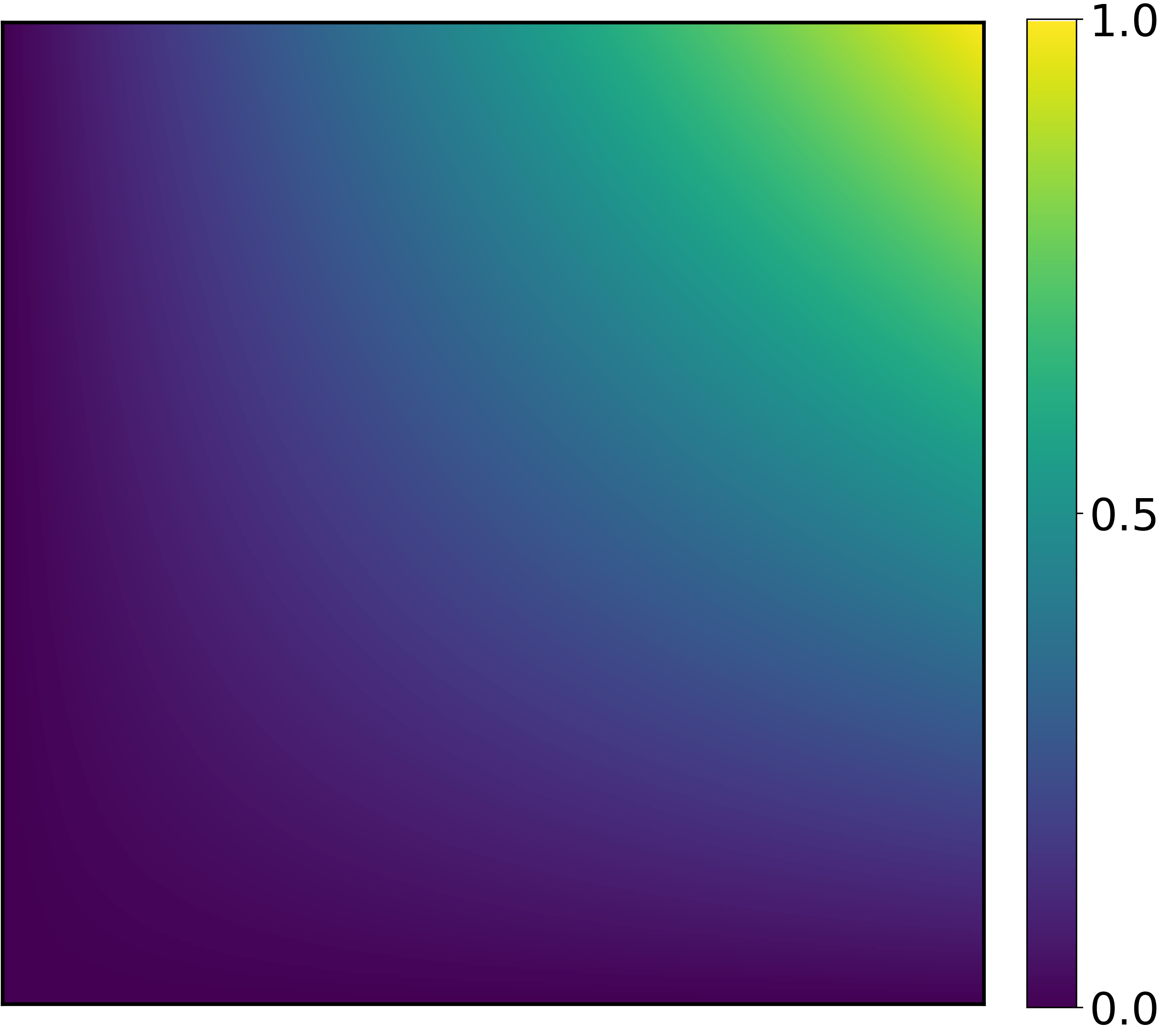}
\end{subfigure}%
\hfill%
\begin{subfigure}[c]{0.2\textwidth}
    \centering
    \textbf{Hölder}\\[0.3em]
    \includegraphics[width=\textwidth]{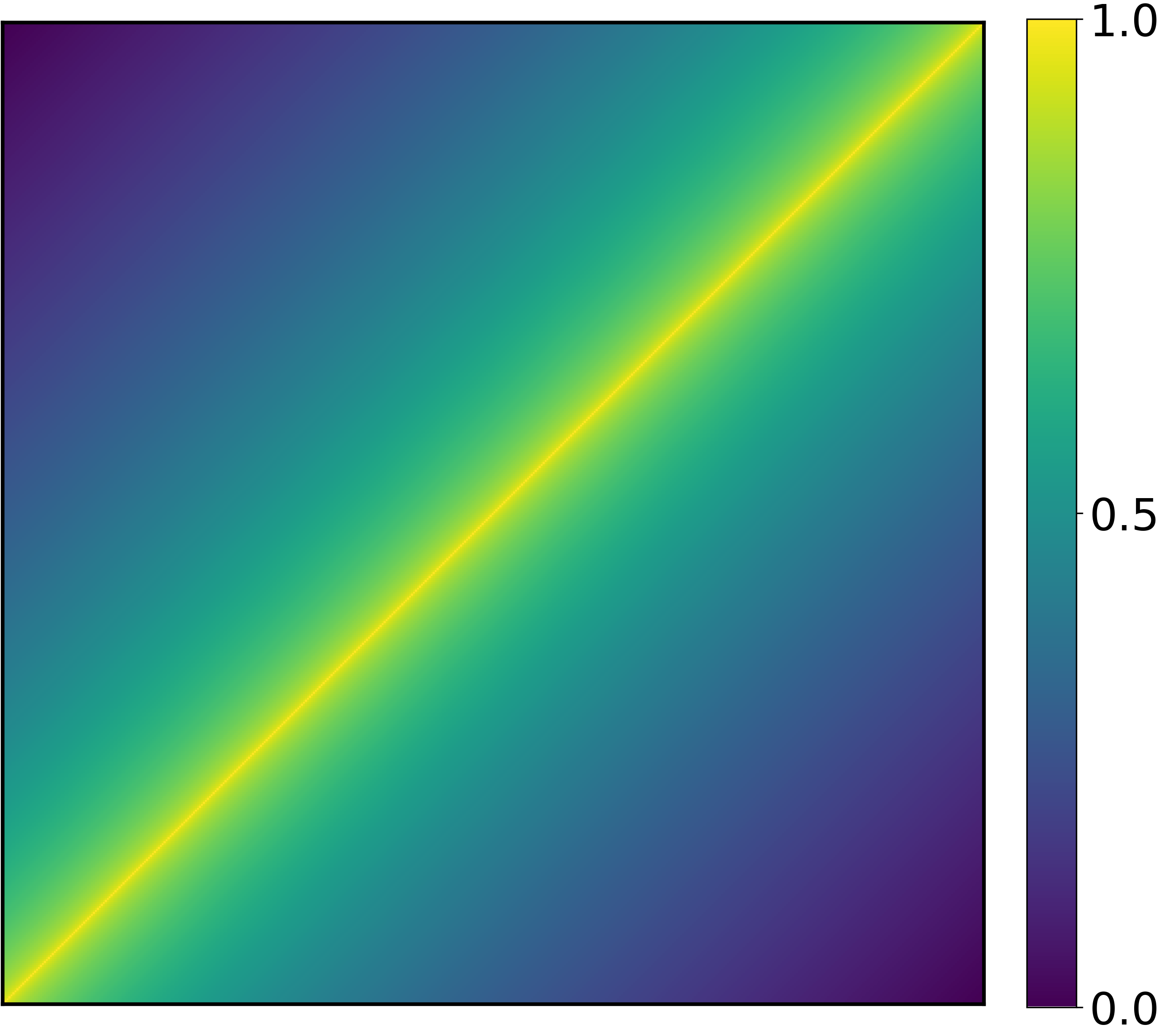}
\end{subfigure}
\caption{Five graphons (left to right): three binary graphons HSBM, Koch, and Hexaflake with increasing boundary complexity; and two weighted graphons: Product (Lipschitz) and Hölder ($\tfrac{1}{2}$-Hölder).}
\label{fig:all_graphons}
\end{figure}

\paragraph{Training Setup}

For each source graph \newrevised{with $n$ nodes} sampled from the graphon, we sample $M$ initial conditions and simulate both the \newrevised{reference dynamics from Table \ref{tab:dynamics}}, with solution denoted by $\teachX{n}(t)$, and the GNDE model, with solution denoted by $\gndeX{n}(t;\theta)$ \newrevised{as in (2.6)}.
The GNDE is trained \newrevised{on a fixed time grid $0 = t_0 < t_1 < \cdots < t_P = T$} by minimizing the following trajectory-wise loss
\[
\mathcal{L}(\theta)
=
\frac{1}{M(P+1)}
\sum_{m=1}^M \sum_{\ell=0}^P
\left\|\newrevised{\bm{X}_n^{(m)}}(t_\ell;\theta) - \teachX{n}^{(m)}(t_\ell)\right\|_{\mathrm{F}}^2.
\]
\newrevised{where $\bm{X}_n^{*(m)}$ refers to the $m$-th trajectory and $\bm{X}_n^{(m)}$ is the GNDE solution of the $m$-th trajectory.} We simulate the dynamics in Table \ref{tab:dynamics} on each generated discrete graph, using $M=10$ random Fourier initial conditions with bandwidth $D=10$, as described in Appendix E.1, with $P=3$ trajectory snapshots. 
All GNDE models are parameterized by spectral GCNs with $K=2$ hops \newrevised{and $L=2$ hidden layers, a choice informed by preliminary experiments and selected to balance approximation accuracy and computational cost. Each model has scalar input and output with hidden dimension $F_1, F_2 = 16$ in the interior layers.} Training is performed using the Adam optimizer \citep{kingma2014adam} with default hyperparameters for $200$ epochs, which was observed as sufficient for convergence.

\begin{table}[H]
\centering
\resizebox{\textwidth}{!}{%
\begin{tabular}{@{}lll@{}}
\toprule
\textbf{Dynamics} & \textbf{Equation on node $i$} & \textbf{Description} \\ 
\midrule
\textbf{Linear Heat} &
$\displaystyle
\frac{d}{dt} [\teachX{n}]_i = \frac{1}{n}
\sum_{j=1}^{n} w_{ij}\,\big([\teachX{n}]_j - [\teachX{n}]_i\big)
$ & \makecell[l]{Classical diffusion (graph Laplacian flow); \\[4pt] models linear heat propagation on graphs.} \\[4pt]

\textbf{Nonlinear Heat} &
$\displaystyle
\frac{d}{dt} [\teachX{n}]_i = \frac{1}{n}
\sum_{j=1}^{n} w_{ij}\, \sin\!\big([\teachX{n}]_j - [\teachX{n}]_i\big)
$ 
& \makecell[l]{Diffusion with nonlinear coupling; \\[4pt]
\newrevised{equivalent to the Kuramoto model} \\[4pt]
\newrevised{with uniform natural frequencies.}} \vspace{0.2cm} \\

\textbf{Allen--Cahn} &
$\displaystyle
\frac{d}{dt} [\teachX{n}]_i =
-\frac{\epsilon^2}{n} \sum_{j=1}^{n} w_{ij}\,\big([\teachX{n}]_j - [\teachX{n}]_i\big)
+ [\teachX{n}]_i - [\teachX{n}]_i^{\,3}
$ & \makecell[l]{Double-well potential dynamics; \\[4pt] describes phase separation and bistability.} \\[4pt]

\textbf{Nonlinear Consensus} &
$\displaystyle
\frac{d}{dt} [\teachX{n}]_i = 
\frac{1}{n} \sum_{j=1}^{n} 
w_{ij}\,
\frac{[\teachX{n}]_j - [\teachX{n}]_i}{1 + ([\teachX{n}]_j - [\teachX{n}]_i)^2}
$ & \makecell[l]{Nonlinear averaging with saturating influence; \\[4pt] models bounded-confidence consensus.} \\
\bottomrule
\end{tabular}%
}
\caption{\newrevised{Ground truth} graph-based dynamical systems, where $w_{ij}$ denotes the adjacency weights and $[\teachX{n}]_i$ represents the node feature at vertex $i$.}
\label{tab:dynamics}
\end{table}

\begin{table}[H]
\centering
\small
\begin{tabular}{@{}llccccc@{}}
\toprule
 & & \multicolumn{3}{c}{\textbf{Binary Graphons}} & \multicolumn{2}{c}{\textbf{Weighted Graphons}} \\
\cmidrule(lr){3-5} \cmidrule(lr){6-7}
\textbf{Dynamics} & \textbf{$n_{\text{val}}$} & \textbf{Hexaflake} & \textbf{HSBM} & \textbf{Koch} & \textbf{Product} & \textbf{H\"older} \\
\midrule
\multirow{4}{*}{\textsc{Allen-Cahn}} 
  & 50  & $0.175 \pm 0.084$ & $0.143 \pm 0.025$ & $0.156 \pm 0.054$ & $0.180 \pm 0.081$ & $0.136 \pm 0.010$ \\
  & 100 & $0.140 \pm 0.016$ & $0.137 \pm 0.017$ & $0.152 \pm 0.054$ & $0.175 \pm 0.078$ & $0.135 \pm 0.010$ \\
  & 150 & $0.155 \pm 0.064$ & $0.138 \pm 0.020$ & $0.152 \pm 0.054$ & $0.173 \pm 0.079$ & $0.135 \pm 0.011$ \\
  & 200 & $0.156 \pm 0.064$ & $0.136 \pm 0.016$ & $0.152 \pm 0.053$ & $0.173 \pm 0.078$ & $0.136 \pm 0.010$ \\
\midrule
\multirow{4}{*}{\textsc{Consensus}} 
  & 50  & $0.115 \pm 0.019$ & $0.110 \pm 0.015$ & $0.150 \pm 0.295$ & $0.190 \pm 0.294$ & $0.102 \pm 0.009$ \\
  & 100 & $0.109 \pm 0.019$ & $0.107 \pm 0.013$ & $0.151 \pm 0.295$ & $0.182 \pm 0.280$ & $0.103 \pm 0.011$ \\
  & 150 & $0.102 \pm 0.019$ & $0.104 \pm 0.008$ & $0.149 \pm 0.296$ & $0.179 \pm 0.281$ & $0.103 \pm 0.012$ \\
  & 200 & $0.103 \pm 0.018$ & $0.106 \pm 0.009$ & $0.152 \pm 0.295$ & $0.177 \pm 0.280$ & $0.102 \pm 0.012$ \\
\midrule
\multirow{4}{*}{\textsc{Linear Heat}} 
  & 50  & $0.048 \pm 0.012$ & $0.054 \pm 0.039$ & $0.028 \pm 0.011$ & $0.048 \pm 0.020$ & $0.047 \pm 0.012$ \\
  & 100 & $0.041 \pm 0.014$ & $0.030 \pm 0.005$ & $0.029 \pm 0.015$ & $0.041 \pm 0.014$ & $0.045 \pm 0.011$ \\
  & 150 & $0.036 \pm 0.013$ & $0.026 \pm 0.003$ & $0.028 \pm 0.011$ & $0.038 \pm 0.013$ & $0.043 \pm 0.012$ \\
  & 200 & $0.036 \pm 0.013$ & $0.025 \pm 0.003$ & $0.026 \pm 0.011$ & $0.036 \pm 0.013$ & $0.040 \pm 0.012$ \\
\midrule
\multirow{4}{*}{\textsc{Nonlin Heat}} 
  & 50  & $0.150 \pm 0.019$ & $0.162 \pm 0.031$ & $0.059 \pm 0.008$ & $0.089 \pm 0.012$ & $0.137 \pm 0.014$ \\
  & 100 & $0.141 \pm 0.012$ & $0.150 \pm 0.026$ & $0.054 \pm 0.009$ & $0.093 \pm 0.014$ & $0.136 \pm 0.012$ \\
  & 150 & $0.123 \pm 0.012$ & $0.149 \pm 0.025$ & $0.054 \pm 0.008$ & $0.092 \pm 0.012$ & $0.137 \pm 0.011$ \\
  & 200 & $0.124 \pm 0.012$ & $0.148 \pm 0.025$ & $0.054 \pm 0.008$ & $0.093 \pm 0.012$ & $0.136 \pm 0.011$ \\
\bottomrule
\end{tabular}
\caption{Mean $\pm$ std.\ of trajectory-wise training error across $10$ training sets, reported for different dynamics and graphons.}
\label{tab:heat_results_num}
\end{table}

\paragraph{Evaluation metric}
Suppose $\hat{\theta}$ is obtained after training.
To quantify how well the GNDE fits the target diffusion dynamics, we define a relative trajectory-wise training error using a discrete approximation of the maximum-in-time $L^2$ error.
Specifically, for each initial condition, we define
\[
\mathsf{Err}_{\mathrm{traj}}^{(m)}
:=
\frac{
\max_{0\le \ell \le Q}
\left(
\bigl\|
\newrevised{\bm{X}_n^{(m)}}(t_\ell;\hat{\theta}) - \teachX{n}^{(m)}(t_\ell)
\bigr\|
\right)
}{
\max_{0\le \ell \le Q}
\left(
\bigl\|
\teachX{n}^{(m)}(t_\ell)
\bigr\|
\right)
}.
\]
Here $\|\cdot\|$ denotes the Euclidean norm, applied componentwise to the node-feature vector at each time point, and $Q$ is chosen sufficiently large so that the discrete maximum over $\{t_\ell\}_{\ell=0}^Q$ provides a reliable approximation of the continuous-time maximum, with $Q=100$ in our experiments. The results are reported for each dynamics and graphon pair in Table \ref{tab:heat_results_num}, and a representative visualization of the resulting learned model is provided in Figure \ref{fig:three_panel}.

\begin{figure}[H]
    \centering

    \includegraphics[width=0.9\textwidth]{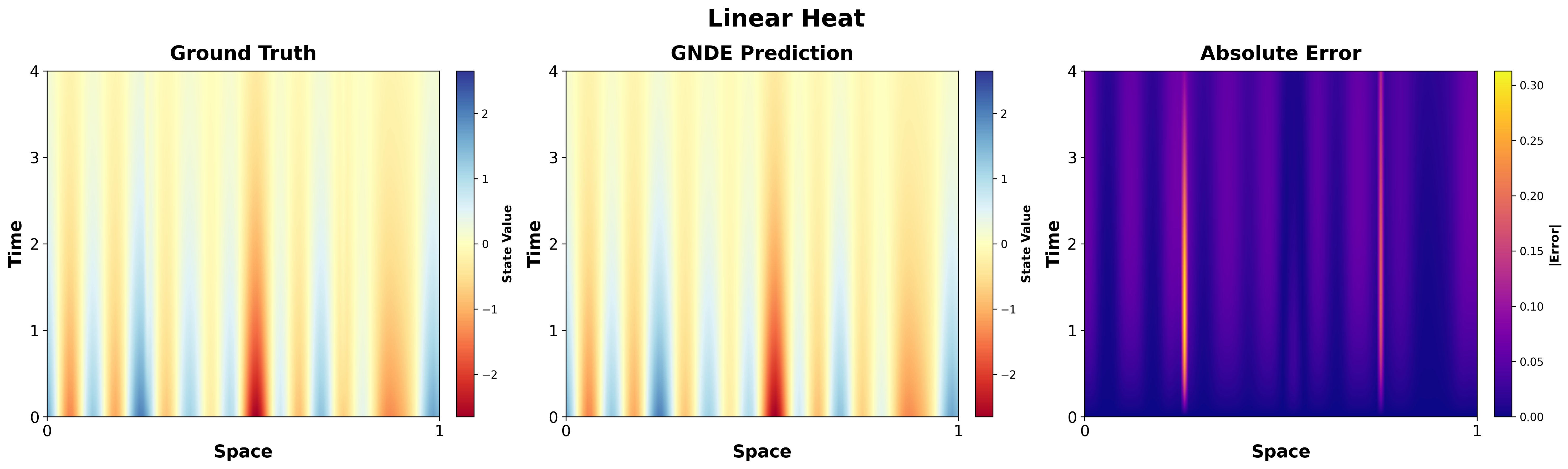}
    \vspace{0.5cm}
    \includegraphics[width=0.9\textwidth]{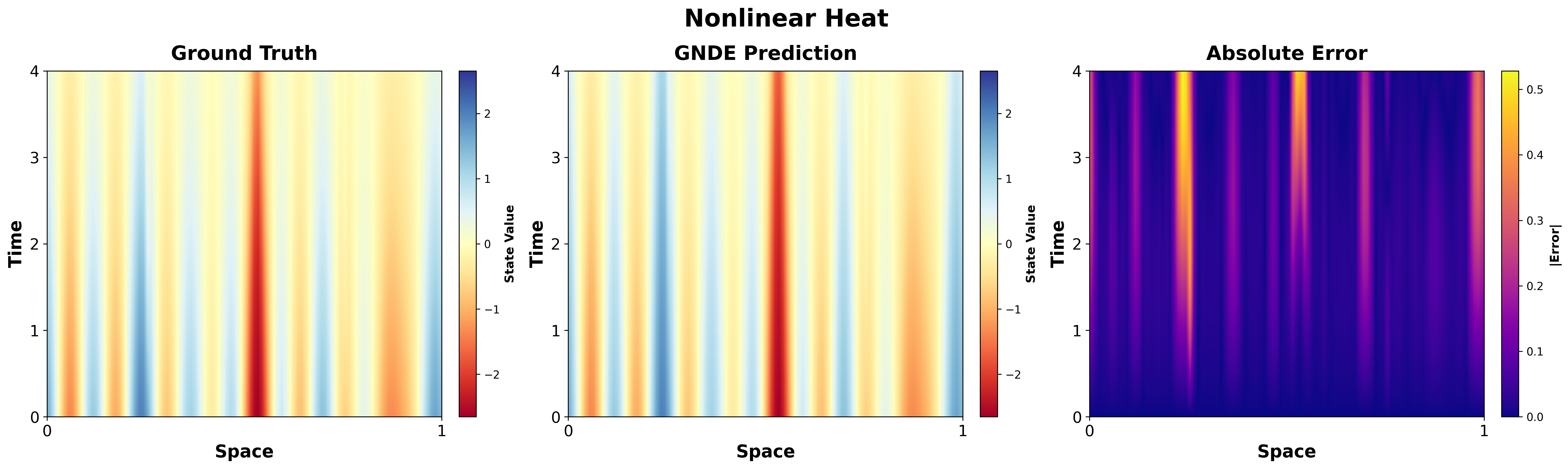}
    \vspace{0.5cm}
    \includegraphics[width=0.9\textwidth]{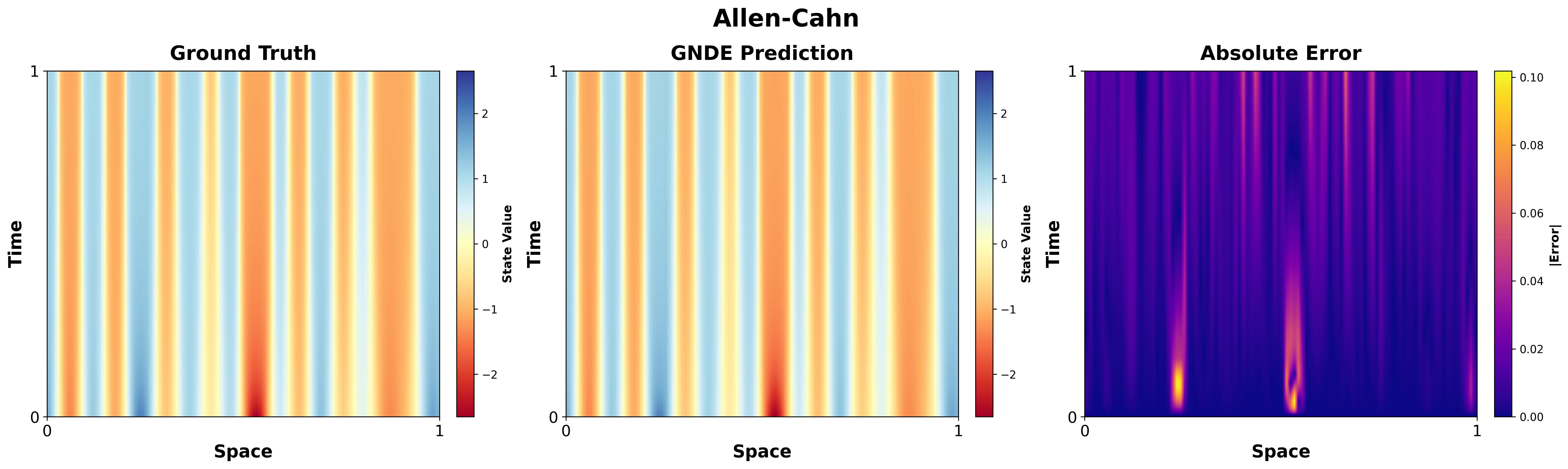}
    \vspace{0.5cm}
    \includegraphics[width=0.9\textwidth]{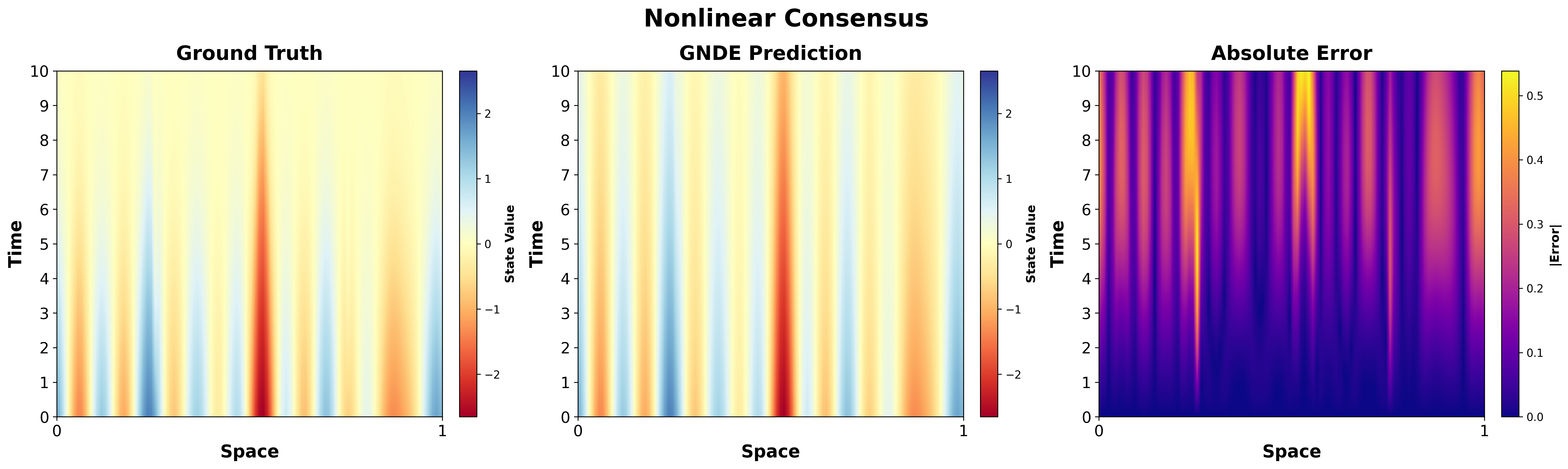}
    \caption{Ground truth, GNDE predicted dynamics, and absolute error plots for the four graph dynamics on a discrete graph of size $100$ generated from the HSBM graphon. Plots show varying patterns of both spatial and temporal error accumulation across dynamics; while some dynamics exhibit short-time fidelity, others such as Allen-Cahn display spatially localized errors from early in the evolution.}
    \label{fig:three_panel}
\end{figure}

\section{Implications and Future Work}
\revised{ We now discuss the implications of our results and directions for future work.}
\paragraph{Computational cost}

For a dense graph with $n$ nodes, the computational cost of a single forward pass of a GNDE with $T$ solver steps, $L$ layers per step, $K$-hop aggregation, and feature dimension $F$ is $\mathcal{O}(n^{2} T L K F)$, which scales quadratically in $n$ and can be prohibitive for large graphs. When graphs are sampled from smooth or binary graphons, the approximation error to the graphon solution decays as $\mathcal{O}(n^{-r})$ for some $r>0$ depending on the graphon’s regularity, in which $r$ is H\"older smoothness exponent for smooth graphons, or relying on boundary complexity for binary graphons. To achieve a target accuracy $\varepsilon$, it suffices to take $n \gtrsim \varepsilon^{-1/r}$, which associates to computational cost $\mathcal{O}(\varepsilon^{-2/r} T L K F)$.

\paragraph{Size transferability bounds} 

Estimates~\eqref{weighted graph: size transferability bounds} and~\eqref{unweighted graph: size transferability bounds} provide quantitative bounds on the discrepancy between GNDE solutions over structurally similar graphs of different sizes $n_1$ and $n_2$, assuming shared convolutional filters. These bounds characterize the size transferability of GNDEs, showing how solution trajectories remain consistent as the graph scales, and highlight the role of graph structure (e.g., graphon regularity) and model smoothness (e.g., convolutional filters, activation functions) in enabling reliable transfer. Our analysis further indicates that transferability becomes more challenging for highly irregular graphs.

\paragraph{Two-scale convergence of discretized GNDEs} Discretized GNDEs can be obtained by applying numerical solvers to GNDEs, resulting in novel constructions of GNNs with residual connections. Despite their practical importance, no convergence analysis for these discretized GNDEs exists in the current literatures. Our convergence results show that GNDE solutions over size‑$n$ graphs converge uniformly in time to a Graphon-NDE solution with rate $\mathcal{O}(n^{-r})$, with $r$ dependent on regularity of graphons. To ensure that such convergence behavior carries over to discretized GNDEs used in practice, we also need to control the numerical solver error. Specifically, if a solver with global error $\mathcal{O}(h^p)$ is used, then to preserve the overall convergence to the graphon limit, we need to require $h^p \ll n^{-r}$. This setup \revised{potentially} reflects a \emph{two-scale convergence}: as both the graph size increases and the time step decreases, the discretized numerical solutions of GNDEs will converge to the Graphon-NDE solution. In practice, this informs the choice of solver: for smooth GNDEs, high-order explicit methods (e.g., RK4) suffice, while stiff dynamics may call for implicit solvers to control long-term error growth. This principle ensures that the discretized model remains consistent across graph sizes and time resolutions. \revised{We leave a rigorous mathematical analysis of two-scale convergence for future studies.}

\paragraph{Future Directions}\label{sec:conclution}
Future work includes extending our trajectory analysis to GNDEs parameterized by \revised{more general} GNN architectures \revised{that admit a well-defined graphon limit}. \revised{For example}, within the spectral GNN framework, \revised{while our results are proved for static graphs, a time-varying graph sequence may be handled under suitable regularity assumptions on the graph evolution}. \revised{One could also consider extensions to general message passing neural networks}. Key challenges  remain in generalizing to non-symmetric architectures, such as attention-based GNNs \citep{velivckovic2017graph,yun2019graph}, which will require novel technical approaches. Additionally, extending our framework to graphs sampled stochastically from underlying graphons represents an important direction, requiring the integration of our trajectory-wise bounds with concentration tools for random graphs.

\section*{Acknowledgment}

The authors thank the anonymous reviewers and the Associate Editor for their careful reading of the manuscript and for their valuable comments and suggestions, which helped improve the presentation of this work. Mingsong Yan acknowledges support from a Simons travel grant. C.~Kulick and S.~Tang were supported by NSF DMS CAREER Grant No.~2340631.

\bibliography{general_bib_file}
\bibliographystyle{plainnat}

\appendix

\section{Proof of Theorem \ref{theorem: well-posedness}}\label{Appendix: well-posedness}
Prior to the detailed proof of Theorem \ref{theorem: well-posedness}, we present several useful observations. Under the assumption \hyperref[AS0]{AS0}, for $T>0$, we define a constant
\begin{equation}\label{def: constant hT}
h_{T}:=\supT\max_{f,g\in[F], \ell\in[L], k\in\mathbb{Z}_K}\left|\hltfgk \right|.
\end{equation}

\begin{lemma}\label{lemma: layer estimate}
Let $T>0$ and $\tX\in C([0,T];L^\infty(I;\mathbb{R}^{1\times F}))$. Suppose that \hyperref[AS0]{AS0} and \hyperref[AS1]{AS1} hold. Then, for $p\in[1,\infty]$, $\ell\in[L]$ and $t\in[0,T]$, it holds that 
    \begin{equation*}
\norm{\tXlt}_{L^p(I;\mathbb{R}^{1\times F})}\leq FKh_T\revised{L_\sigma}\norm{\tX^{(\ell-1,t)}}_{L^p(I;\mathbb{R}^{1\times F})},
    \end{equation*}
    where $h_T$ is defined in \eqref{def: constant hT}. 
\end{lemma}
\begin{proof}
    Note that the updating rule of Graphon-NN gives 
    \begin{equation*}
\tXflt=\sigma\parens{\sum_{g=1}^F\sum_{k=0}^{K-1}\hltfgk T_\tW^k \tXgglt},\quad f\in[F], \ell\in[L], t\in[0,T]. 
    \end{equation*}
    It follows that  
     \begin{equation*}
\norm{\tXflt}_{L^p(I)}\leq h_T\revised{L_\sigma}\parens{\sum_{k=0}^{K-1} \norm{T_\tW}_{L^p(I)\to L^p(I)}^k}\norm{\sum_{g=1}^F \tXgglt}_{L^p(I)}\leq h_T\revised{L_\sigma}K\sqrt{F}\norm{\tX^{(\ell-1,t)}}_{L^p(I;\mathbb{R}^{1\times F})},
    \end{equation*}
    in which the first inequality is due to \hyperref[AS0]{AS0}, \hyperref[AS1]{AS1} and triangle inequality; the second is according to the fact of $\norm{T_\tW}_{L^p(I)\to L^p(I)}\leq\norm{\tW}_{L^\infty(I^2)}\leq 1$ and the norm defined in $L^p(I;\mathbb{R}^{1\times F})$. The desired result immediately follows by rewriting the norm of $\tXlt$. 
\end{proof}
\begin{proposition}\label{proposition: bound phi(X1) - phi(X2) infinity norm}
Suppose that {\color{black}\hyperref[AS0]{AS0} and \hyperref[AS1]{AS1}} hold. Let $T>0$ and $\tX,\widetilde{\tX}\in C([0,T];L^{\infty}(I; \mathbb{R}^{1\times F}))$. Then for all $t\in[0,T]$, it holds that 
\begin{align*}
&\norm{\GraphonNN(\tX(\cdot,t);\tW,\tH(t)) -\GraphonNN(\widetilde{\tX}(\cdot,t);\tW,\tH(t))}_{L^{\infty}(I;\mathbb{R}^{1\times F})} \leq  (FKh_T\revised{L_\sigma})^{L} \norm{\tX(\cdot,t)-\widetilde{\tX}(\cdot,t)}_{L^\infty(I;\mathbb{R}^{1\times F})}.
\end{align*} 

\end{proposition}
\begin{proof}  
According to the Lipschitz continuity of activation function $\sigma$, similarly to the proof of Lemma \ref{lemma: layer estimate} with $p=\infty$, we have 
    \begin{equation}\label{in the proof: layer estimate diff}
\norm{\tXlt-\widetilde{\tX}^{(\ell,t)}}_{L^\infty(I;\mathbb{R}^{1\times F})}\leq FKh_T\revised{L_\sigma}\norm{\tX^{(\ell-1,t)}-\widetilde{\tX}^{(\ell-1,t)}}_{L^\infty(I;\mathbb{R}^{1\times F})}.
    \end{equation}
    Recall the notations $\tX(\cdot,t)=\tX^{(0,t)}$, $\GraphonNN(\tX(\cdot,t);\tW,\tH(t))=\tX^{(L,t)}$ (similar for $\widetilde{\tX}$). The desired result follows from recursively applying \eqref{in the proof: layer estimate diff}. 
\end{proof}

\begin{proof}[Proof of Theorem \ref{theorem: well-posedness}] The proof is based on the Banach contraction mapping principle, \revised{adapting the argument of \citet{medvedev2014nonlinear} to our setting of vector-valued dynamics with time-dependent coefficients}. Let $T>0$ be arbitrary but fixed, and $0<\tau<\frac{1}{2(FKh_T\revised{L_\sigma})^{L}}$. We define a \revised{closed subset} $\mathcal{S}_{\tZ}$ of the Banach space $C([0,\tau];L^{\infty}(I;\mathbb{R}^{1\times F}))$, associated with $\tau$, by
\begin{equation*}
\mathcal{S}_{\tZ}:=\left\{\tX:\tX\in C([0,\tau];L^{\infty}(I;\mathbb{R}^{1\times F})),\tX(\cdot,0)=\tZ\right\},
\end{equation*}
\revised{equipped with the norm induced from $C([0,\tau];L^{\infty}(I;\mathbb{R}^{1\times F}))$.} Moreover, we define an integral operator $\mathcal{K}:\mathcal{S}_\tZ\to\mathcal{S}_\tZ$ by 
\begin{align}\label{fixpoint}
[\mathcal{K}\tX](u,t):=\tZ(u)+\int_{0}^{t}\GraphonNN(\tX(u,s);\tW,\tH(s))ds.
\end{align} 
It follows that we can rewrite the initial value problem \eqref{Graphon-NDE} as the fixed point equation
$\tX=\mathcal{K}\tX$. We show below that the operator $\mathcal{K}$ is a contraction. For any $\tX,\widetilde{\tX} \in \mathcal{S}_{\tZ}$, according to the definition of norm in $C([0,\tau];L^\infty(I;\mathbb{R}^{1\times F}))$, we have 
\begin{align}
\|\mathcal{K}\tX-\mathcal{K}\widetilde{\tX}\|_{\mathcal{S}_{\tZ}}&= \sup_{t\in [0,\tau]}\|\mathcal{K}\tX(\cdot,t)-\mathcal{K}\widetilde{\tX}(\cdot,t)\|_{L^\infty(I;\mathbb{R}^{1\times F})}\nonumber\\ 
&=\sup_{t\in [0,\tau]}\left\|\int_{0}^{t} \GraphonNN(\tX(\cdot,s);\tW,\tH(s))-\GraphonNN(\widetilde{\tX}(\cdot,s);\tW,\tH(s))\ ds\right\|_{L^\infty(I;\mathbb{R}^{1\times F})}\nonumber\\
&\leq \tau \sup_{t\in [0,\tau]} \left\|\GraphonNN(\tX(\cdot,t);\tW,\tH(t))-\GraphonNN(\widetilde{\tX}(\cdot,t);\tW,\tH(t))\right\|_{L^{\infty}(I;\mathbb{R}^{1\times F})}.\label{first estimate of KX1-KX2 in proof}
\end{align}
It follows from Lemma \ref{proposition: bound phi(X1) - phi(X2) infinity norm} that 
\begin{equation*}
\left\|\GraphonNN(\tX(\cdot,t);\tW,\tH(t))-\GraphonNN(\widetilde{\tX}(\cdot,t);\tW,\tH(t))\right\|_{L^{\infty}(I;\mathbb{R}^{1\times F})}\leq (FKh_T\revised{L_\sigma})^L\|\tX(\cdot,t)-\widetilde{\tX}(\cdot,t)\|_{L^{\infty}(I;\mathbb{R}^{1\times F})}.
\end{equation*}
By substituting the above estimate into \eqref{first estimate of KX1-KX2 in proof}, we obtain that 
\begin{align*}
    \|\mathcal{K}\tX-\mathcal{K}\widetilde{\tX}\|_{\mathcal{S}_{\tZ}} &\leq \tau (FKh_T\revised{L_\sigma})^{L}\sup_{t\in [0,\tau]}  \|\tX(\cdot,t)-\widetilde{\tX}(\cdot,t)\|_{L^{\infty}(I;\mathbb{R}^{1\times F})}\\
    &=\tau (FKh_T\revised{L_\sigma})^{L} \|\tX-\widetilde{\tX}\|_{\mathcal{S}_{\tZ}} \leq \frac{1}{2} \|\tX-\widetilde{\tX}\|_{\mathcal{S}_{\tZ}}
\end{align*}
where the last inequality follows from the definition of $\tau$. Therefore, the operator $\mathcal{K}$ is a contraction. By the Banach contraction mapping principle, there exists a unique \revised{local} solution $\widehat{\tX} \in\mathcal{S}_{\tZ}$ \revised{\revised{defined on the time interval $[0, \tau]$}} for the initial value problem \eqref{Graphon-NDE}. Taking $\widehat{\tX}(\tau)$ as the initial condition, we repeat the argument to extend the solution to $[0,2\tau]$. In such a way, we can keep doing until the solution extends to $[0,T]$, and get a unique solution $\tX\in C([0,T];L^\infty(I;\mathbb{R}^{1\times F}))$. According to {\color{black}\hyperref[AS0]{AS0} and \hyperref[AS1]{AS1}}, it follows that $\GraphonNN(\tX(u,\cdot);\tW,\tH(\cdot))$ is continuous, that is, the integrand in \eqref{fixpoint} is continuous. Therefore, by \revised{the} fundamental theorem of calculus, we see that $\mathcal{K}\tX$ is continuously differentiable about the second variable $t$. As $\mathcal{K}\tX=\tX$, we conclude that $\tX\in C^1([0,T];L^\infty(I;\mathbb{R}^{1\times F}))$. This completes the proof.
\end{proof}

\section{Stability Analysis of Graphon-NDEs}
To lay a foundation for the subsequent proofs of the convergence result (Theorem \ref{theorem: convergence of solutions}) and also the convergence rate results (Theorems \ref{theorem: rate of Lipschitz} and \ref{theorem: rate of simple graph}), this section focuses on the stability analysis of Graphon-NDEs. We proceed with several technical lemmas. 

\begin{lemma}\label{lemma: difference of powers}
Let $T_1$ and $T_2$ be two bounded linear operators on $L^2(I)$. Let $k$ be a given positive integer. If $\|T_1\|_{L^2(I)\to L^2(I)}\leq 1$ and $\|T_2\|_{L^2(I)\to L^2(I)}\leq 1$, then $\left\|T_1^k-T_2^k\right\|_{L^2(I)\to L^2(I)}\leq k\|T_1-T_2\|_{L^2(I)\to L^2(I)}$.
\end{lemma}
\revised{
\begin{proof}
The result immediately follows by noting that $T_1^k-T_2^k=\sum_{j=0}^{k-1} T_1^{k-j-1}(T_1-T_2)T_2^{j}$. 
\end{proof}
}

\begin{lemma}[Stability of Graphon-NNs]\label{lemma: transferbility of Graphon NN}
Let $T>0$, $\tX,\widetilde{\tX}\in C([0,T];L^{\infty}(I; \mathbb{R}^{1\times F}))$, and graphons $\tW,\widetilde{\tW}$. If {\color{black}\hyperref[AS0]{AS0} and \hyperref[AS1]{AS1}} hold, then for any $t\in[0,T]$, it holds that 
\begin{align*}
&\norm{\GraphonNN\parens{\widetilde{\tX}(\cdot,t);\widetilde{\tW},\tH(t)}-\GraphonNN\parens{\tX(\cdot,t);\tW,\tH(t)}}_{L^2(I;\mathbb{R}^{1\times F})} \\
\leq &\parens{FKh_T\revised{L_\sigma}}^{L}\parens{\norm{\widetilde{\tX}(\cdot,t)-\tX(\cdot,t)}_{L^2(I;\mathbb{R}^{1\times F})} + LK\norm{T_{\widetilde{\tW}}-T_{\tW}}_{L^2(I)\to L^2(I)}\norm{\tX}_{C([0,T];L^2(I;\mathbb{R}^{1\times F}))}}.
\end{align*}

\end{lemma}

\begin{proof}
    Recall that for $f\in[F], \ell\in[L], t\in[0,T]$, the updating rule of Graphon-NN gives
    \begin{equation*}
        \tXflt=\sigma\parens{\sum_{g=1}^F\sum_{k=0}^{K-1}\hltfgk T_\tW^k \tXgglt},\quad\widetilde{\tX}_f^{(\ell,t)}=\sigma\parens{\sum_{g=1}^F\sum_{k=0}^{K-1}\hltfgk T_{\widetilde{\tW}}^k \widetilde{\tX}_g^{(\ell-1,t)}}.
    \end{equation*}
    \newrevised{By triangle inequality, we have $\norm{\widetilde{\tX}_f^{(\ell,t)}-\tXflt}_{L^2(I)}\leq \Delta_1+\Delta_2$, where
    \begin{align*}
\Delta_1&:=\norm{\sigma\parens{\sum_{g=1}^F\sum_{k=0}^{K-1}\hltfgk T_{\widetilde{\tW}}^k \widetilde{\tX}_g^{(\ell-1,t)}} - \sigma\parens{\sum_{g=1}^F\sum_{k=0}^{K-1}\hltfgk T_{\widetilde{\tW}}^k \tX_g^{(\ell-1,t)}}}_{L^2(I)},\\
\Delta_2&:=\norm{\sigma\parens{\sum_{g=1}^F\sum_{k=0}^{K-1}\hltfgk T_{\widetilde{\tW}}^k \tX_g^{(\ell-1,t)}} - \sigma\parens{\sum_{g=1}^F\sum_{k=0}^{K-1}\hltfgk T_\tW^k \tXgglt} }_{L^2(I)}.
    \end{align*}
    Then by similar argument as in the proof of Lemma \ref{lemma: layer estimate}, we get
    \begin{align*}
    \Delta_1&\leq\sqrt{F}Kh_T{L_\sigma}\norm{\widetilde{\tX}^{(\ell-1,t)}-\tX^{(\ell-1,t)}}_{L^2(I;\mathbb{R}^{1\times F})},\\
\Delta_2&\leq\sqrt{F}h_T{L_\sigma}\parens{\sum_{k=0}^{K-1}\norm{T_{\widetilde{\tW}}^k-T_\tW^k}_{L^2(I)\to L^2(I)}}\norm{\tX^{(\ell-1,t)}}_{L^2(I;\mathbb{R}^{1\times F})}.
    \end{align*}
    Hence, we obtain that }
    \begin{align*}
        \norm{\widetilde{\tX}_f^{(\ell,t)}-\tXflt}_{L^2(I)}\leq &\sqrt{F}Kh_T\revised{L_\sigma}\norm{\widetilde{\tX}^{(\ell-1,t)}-\tX^{(\ell-1,t)}}_{L^2(I;\mathbb{R}^{1\times F})}\\
        &+\sqrt{F}h_T\revised{L_\sigma}\parens{\sum_{k=0}^{K-1}\norm{T_{\widetilde{\tW}}^k-T_\tW^k}_{L^2(I)\to L^2(I)}}\norm{\tX^{(\ell-1,t)}}_{L^2(I;\mathbb{R}^{1\times F})}.
    \end{align*}
    It follows from Lemma \ref{lemma: difference of powers} that 
    $$\sum_{k=0}^{K-1}\norm{T_{\widetilde{\tW}}^k-T_\tW^k}_{L^2(I)\to L^2(I)}\leq K^2\norm{T_{\widetilde{\tW}}-T_\tW}_{L^2(I)\to L^2(I)}.$$
    Therefore, 
    \begin{align*}
        \norm{\widetilde{\tX}^{(\ell,t)}-\tXlt}_{L^2(I;\mathbb{R}^{1\times F})}\leq& FKh_T\revised{L_\sigma}\norm{\widetilde{\tX}^{(\ell-1,t)}-\tX^{(\ell-1,t)}}_{L^2(I;\mathbb{R}^{1\times F})}\\
        &+FK^2h_T\revised{L_\sigma}\norm{T_{\widetilde{\tW}}-T_\tW}_{L^2(I)\to L^2(I)}\norm{\tX^{(\ell-1,t)}}_{L^2(I;\mathbb{R}^{1\times F})}.
    \end{align*}
    Then a recursion argument gives 
    \begin{align*}
        \norm{\widetilde{\tX}^{(L,t)}-\tX^{(L,t)}}_{L^2(I;\mathbb{R}^{1\times F})}&\leq\parens{FKh_T\revised{L_\sigma}}^L\norm{\widetilde{\tX}^{(0,t)}-\tX^{(0,t)}}_{L^2(I;\mathbb{R}^{1\times F})} \\
        +FK^2h_T&\revised{L_\sigma}\norm{T_{\widetilde{\tW}}-T_\tW}_{L^2(I)\to L^2(I)}\sum_{\ell=0}^{L-1}\parens{FKh_T\revised{L_\sigma}}^{L-1-\ell}\norm{\tX^{(\ell,t)}}_{L^2(I;\mathbb{R}^{1\times F})}.
    \end{align*}
    Note that by Lemma \ref{lemma: layer estimate}, we have $\norm{\tX^{(\ell,t)}}_{L^2(I;\mathbb{R}^{1\times F})}\leq (FKh_T\revised{L_\sigma})^\ell\norm{\tX^{(0,t)}}_{L^2(I;\mathbb{R}^{1\times F})}$. Hence, 
     \begin{align*}
        \norm{\widetilde{\tX}^{(L,t)}-\tX^{(L,t)}}_{L^2(I;\mathbb{R}^{1\times F})}&\leq\parens{FKh_T\revised{L_\sigma}}^L\norm{\widetilde{\tX}^{(0,t)}-\tX^{(0,t)}}_{L^2(I;\mathbb{R}^{1\times F})} \\
        &+LK\parens{FKh_T\revised{L_\sigma}}^{L}\norm{T_{\widetilde{\tW}}-T_\tW}_{L^2(I)\to L^2(I)}\norm{\tX^{(0,t)}}_{L^2(I;\mathbb{R}^{1\times F})}.
    \end{align*}
    Note that $\tX^{(0,t)}=\tX(\cdot,t)$, $\tX^{(L,t)}=\GraphonNN\parens{\tX(\cdot,t);\tW,\tH(t)}$ (similar for $\widetilde{\tX}$) and norm of $\tX$ in $C([0,T];L^2(\mathbb{R}^{1\times F}))$ is defined as the supremum of $\norm{\tX(\cdot,t)}_{L^2(I;\mathbb{R}^{1\times F})}$ about $t\in[0,T]$. Therefore, the above inequality implies the desired result. 
\end{proof}

The following result is a special case of \citet{perov1959k} (also see Theorem 21 in \citet{dragomir2003some}).
\begin{lemma}[Generalized Gr\"onwall’s inequality]\label{lemma: generalized Gronwall's inequality}
Let $a$, $b$ and $c$ be non-negative constants. Let $u(t)$ be a non-negative function that satisfies the integral inequality $u(t) \leq c+\int_{0}^t\left(a u(s)+b u^{\frac{1}{2}}(s)\right) ds$,
then we have $u(t)\leq \left(c^{\frac{1}{2}}\mathrm{exp}(at/2)+\frac{\mathrm{exp}(at/2)-1}{a}b\right)^2$.
\end{lemma}

Now given a sequence of graphons $\{\tW_n\}_{n=1}^\infty$ and (bounded) input feature functions $\{\tZ_n\}_{n=1}^\infty$, we consider the following Graphon-NDEs
\begin{align}\label{induced Graphon-NDE simplified notation}
    \begin{split}
        \frac{\partial}{\partial t} \tX_n(u,t) &= \GraphonNN(\tX_n(u, t);\tW_n, \tH(t)),\\
        \tX_n(u,0) &=  \tZ_n(u).
    \end{split}
\end{align}
We note that Theorem \ref{theorem: well-posedness} guarantees the existence and uniqueness of the solution $\tX_n$ of \eqref{induced Graphon-NDE simplified notation}. We establish in the following that the error between solutions of \eqref{Graphon-NDE} and \eqref{induced Graphon-NDE simplified notation} is bounded above by a linear combination of the initial feature error and graphon error.

\begin{theorem}[Stability of Graphon-NDEs]\label{theorem: X-Xn leq G-Gn + TW-TWn}
Suppose that {\color{black}\hyperref[AS0]{AS0} and \hyperref[AS1]{AS1}} hold. Let $\tX$ and $\tX_n$ denote the solutions of \eqref{Graphon-NDE} and \eqref{induced Graphon-NDE simplified notation}, respectively. Then it holds that
\begin{equation}\label{norm difference of XT and XnT}
\|\tX_n-\tX\|_{C([0,T];L^2(I;\mathbb{R}^{1\times F}))} \leq P\|\tZ_n-\tZ\|_{L^2(I;\mathbb{R}^{1\times F})}+Q\|T_{\tW_n}-T_{\tW}\|_{L^2(I)\to L^2(I)},
\end{equation}
where 
\begin{equation}\label{constants P and Q}
    P:=\exp\parens{T\parens{FKh_T\revised{L_\sigma}}^L}, \quad Q:=(P-1)LK\norm{\tX}_{C([0,T];L^2(I;\mathbb{R}^{1\times F}))}.
\end{equation}

\end{theorem}

\begin{proof}
Denote $\Delta = \tX_n - \tX$. Taking the difference between \eqref{induced Graphon-NDE simplified notation} and \eqref{Graphon-NDE}, we have 
\begin{align*}
\begin{split}
    \frac{\partial}{\partial t}\Delta(u,t) &= \GraphonNN(\tX_n(u,t);\tW_n,\tH(t)) - \GraphonNN(\tX(u,t);\tW,\tH(t)),\\
    \Delta(u,0) &= \tZ_n(u) - \tZ(u).
\end{split}
\end{align*}
It follows that 
\begin{align*}
&\frac{1}{2} \frac{d}{dt} \|\Delta(\cdot,t)\|_{L^2(I;\mathbb{R}^{1\times F})}^2=\left\lvert  \int_{I}   \frac{\partial \Delta(u,t)}{\partial t}\parens{\Delta(u,t)}^{\top}  du \right\rvert\\ 
=\ &  \left| \int_I  \parens{\GraphonNN\parens{\tX_n(u,t);\tW_n,\tH(t)}-\GraphonNN\parens{\tX(u,t);\tW,\tH(t)}} \parens{\Delta(u,t)}^{\top} du \right| \\
\leq\ &   \norm{\GraphonNN(\tX_n(\cdot,t);\tW_n,\tH(t))-\GraphonNN(\tX(\cdot,t);\tW,\tH(t))}_{L^2(I;\mathbb{R}^{1\times F})}\|\Delta(\cdot,t)\|_{L^2(I;\mathbb{R}^{1\times F})}.
\end{align*} 
According to Lemma \ref{lemma: transferbility of Graphon NN}, we have 
\begin{align*}
&\norm{\GraphonNN(\tX_n(\cdot,t);\tW_n,\tH(t))-\GraphonNN(\tX(\cdot,t);\tW, \tH(t))}_{L^2(I;\mathbb{R}^{1\times F})}\\
&\leq \underbrace{\parens{FKh_T\revised{L_\sigma}}^{L}}_{\text{denoted by }a/2} \norm{\Delta(\cdot,t)}_{L^2(I;\mathbb{R}^{1\times F})} + \underbrace{LK\parens{FKh_T\revised{L_\sigma}}^L\norm{T_{\tW_n}-T_{\tW}}_{L^2(I)\to L^2(I)}\norm{\tX}_{C([0,T];L^2(I;\mathbb{R}^{1\times F}))}}_{\text{denoted by }b/2}.
\end{align*}
Let $\delta(t):=\norm{\Delta(\cdot,t)}_{L^2(I;\mathbb{R}^{1\times F})}^2$. Then the above estimates lead to 
\begin{equation*}
    \begin{split}
        \frac{d}{dt}\delta(t)&\leq a\delta(t) + b\sqrt{\delta(t)},\\
        \delta(0)&=\norm{\tZ_n-\tZ}_{L^2(I;\mathbb{R}^{1\times F})}^2.
    \end{split}
\end{equation*}
Let $s\in[0,T]$ be arbitrary but fixed. \revised{We integrate with respect to $t$ over $[0,s]$}, and get $$\delta(s)\leq \delta(0) + \int_0^{s} \left(a\delta(t) + b\sqrt{\delta(t)}\right)dt.$$ We then apply the generalized Gr\"onwall's inequality (Lemma \ref{lemma: generalized Gronwall's inequality}), and get
\begin{equation*}
    \delta(s)\leq \left(\sqrt{\delta(0)}\mathrm{exp}(as/2)+\frac{\mathrm{exp}(as/2)-1}{a}b\right)^2.
\end{equation*}
By noting $s\leq T$ and plugging definitions of $a$, $b$ and $\delta$ into the above inequality, we obtain 
\begin{align*}
\norm{\Delta(\cdot,s)}_{L^2(I;\mathbb{R}^{1\times F})}&\leq P\norm{\tZ_n-\tZ}_{L^2(I;\mathbb{R}^{1\times F})}+ Q\norm{T_{\tW_n}-T_{\tW}}_{L^2(I)\to L^2(I)},
\end{align*}
with $P$ and $Q$ defined in \eqref{constants P and Q}. Since $s$ is arbitrary in $[0,T]$, we take the supremum about $s$ over $[0,T]$ for the above inequality, and immediately get \eqref{norm difference of XT and XnT} by recalling the norm defined in $C([0,T];L^2(I;\mathbb{R}^{1\times F}))$. 
\end{proof}

\section{Proof of Theorem \ref{theorem: convergence of solutions}}\label{appendix: Proof of convergence of Graphon-NDEs}

\begin{proof}[Proof of Theorem \ref{theorem: convergence of solutions}]
By the assumption of $\{(\mathcal{G}_n,\mZ_{\mathcal{G}_n})\}_{n=1}^\infty$ converging to $(\tW, \tZ)$ in the sense of Definition \ref{def: graph node feature converges -- formal one}, there exists a sequence $\{\pi_n\}_{n=1}^\infty$ of permutations such that 
\begin{equation}\label{temp graphon and G converge}
    \lim_{n\to\infty}\|\tW_{\pi_n(\mathcal{G}_n)}-\tW\|_{\square}=0,\quad\lim_{n\to\infty}\|\tZ_{\pi_n(\mathcal{G}_n)}-\tZ\|_{L^2(I;\mathbb{R}^{1\times F})}=0.
\end{equation}
We denote $\tW_n:=\tW_{\pi_n(\mathcal{G}_n)}$ and $\tZ_n:=\tZ_{\pi_n(\mathcal{G}_n)}$.
It is known (Lemma E.6. in \citet{janson2010graphons}) that $\lim_{n\to\infty}\|\tW_n-\tW\|_\square=0$ if and only if $\lim_{n\to\infty}\|T_{\tW_n}-T_{\tW}\|_{L^2(I)\to L^2(I)}=0$. Therefore, \eqref{temp graphon and G converge} implies 
\begin{equation}\label{integral opeator and Gn converge}
    \lim_{n\to\infty}\|T_{\tW_n}-T_{\tW}\|_{L^2(I)\to L^2(I)}=0,\quad \lim_{n\to\infty}\|\tZ_n-\tZ\|_{L^2(I;\mathbb{R}^{1\times F})}=0.
\end{equation}
Then the desired result immediately follows from Theorem \ref{theorem: X-Xn leq G-Gn + TW-TWn}. 
\end{proof}

\section{Proof of Theorems \ref{theorem: rate of Lipschitz} and \ref{theorem: rate of simple graph}}\label{Appendix: proof of convergence rates}

\begin{proof}[Proof of Theorem \ref{theorem: rate of Lipschitz}]
Recall that $u_i:=(i-1)/n$, $I_i:=[u_i,u_{i+1})$, for each $i\in[n]$. According to \revised{definition \eqref{induced graphon Wn} of $\tW_n$} with \eqref{coeff W for weighted graphs}, we have 
$$\|\tW-\tW_n\|_{L^2({I^2})}^2=\sum_{i,j\in[n]}\int_{I_i\times I_j}\left|\tW(u,v)-\tW(u_i,u_j)\right|^2dudv.$$ According to {\color{black}\hyperref[AS2]{AS2}}, we obtain that 
\begin{equation}\label{upper bound of norm of W-Wn first estimate}
    \|\tW-\tW_n\|_{L^2({I^2})}^2\leq A_1^2\sum_{i,j\in[n]}\int_{I_i\times I_j}\left(|u-u_i|+|v-u_j|\right)^{2\revised{\alpha_1}}dudv.
\end{equation}
For each $i,j\in[n]$, direct computation gives $$\int_{I_i\times I_j}\left(|u-u_i|+|v-u_j|\right)^{2\revised{\alpha_1}}dudv=\frac{2^{2\revised{\alpha_1}+2}-2}{(2\revised{\alpha_1}+1)(2\revised{\alpha_1}+2)}\frac{1}{n^{2\revised{\alpha_1}+2}},$$ which combining with \eqref{upper bound of norm of W-Wn first estimate} gives 
\begin{equation}\label{final norm of W - Wn}
\|\tW-\tW_n\|_{L^2({I^2})}^2 \leq A_1^2\frac{2^{2\revised{\alpha_1}+2}-2}{(2\revised{\alpha_1}+1)(2\revised{\alpha_1}+2)}\frac{1}{n^{2\revised{\alpha_1}}}.
\end{equation}
Denote $\tZ=[Z_f:f\in[F]]$ and $\tZ_n=[(Z_n)_f:f\in[F]]$. According to definition $\tZ_n$ of \eqref{induced graph features Zn} with \eqref{coeff G for weighted graphs}, we have
\begin{equation}\label{L2 norm square of g - gn}
    \|\tZ-\tZ_n\|_{L^2(I;\mathbb{R}^{1\times F})}^2=\sum_{f\in[F]}\|Z_f-(Z_n)_f\|_{L^2(I)}^2=\sum_{f\in[F]}\sum_{j\in[n]}\int_{I_j} |Z_f(u)-Z_f(u_j)|^2du.
\end{equation}
It follows from {\color{black}\hyperref[AS3]{AS3}} that for each $f\in[F]$ and $j\in[n]$, 
$$\int_{I_j} |Z_f(u)-Z_f(u_j)|^2du\leq A_2^2\int_{I_j} (u-u_j)^{2\revised{\alpha_2}}du=\frac{A_2^2}{\revised{2\alpha_2+1}}\frac{1}{n^{\revised{2\alpha_2+1}}}.$$
Therefore, from \eqref{L2 norm square of g - gn}, we get 
\begin{equation}\label{final norm of g - gn}
    \|\tZ-\tZ_n\|_{L^2(I;\mathbb{R}^{1\times F})}^2\leq \frac{A_2^2F}{\revised{2\alpha_2+1}}\frac{1}{n^{\revised{2\alpha_2}}}.
\end{equation}
Recall we have established in Theorem \ref{theorem: X-Xn leq G-Gn + TW-TWn} that 
$$
\|\tX_n-\tX\|_{C([0,T];L^2(I;\mathbb{R}^{1\times F}))} \leq P\|\tZ_n-\tZ\|_{L^2(I;\mathbb{R}^{1\times F})}+Q\|T_{\tW_n}-T_{\tW}\|_{L^2(I)\to L^2(I)},
$$
which combining with estimates \eqref{final norm of W - Wn} and \eqref{final norm of g - gn} and the fact of 
$$\norm{T_{\tW_n}-T_\tW}_{L^2(I)\to L^2(I)}\leq \norm{\tW_n-\tW}_{L^2(I^2)},$$ 
further implies 
$$
\|\tX_n-\tX\|_{C([0,T];L^2(I;\mathbb{R}^{1\times F}))} \leq PA_2\sqrt{\frac{F}{\revised{2\alpha_2+1}}}\frac{1}{n^{\revised{\alpha_2}}}+Q  A_1\sqrt{\frac{2^{2\revised{\alpha_1}+2}-2}{(2\revised{\alpha_1}+1)(2\revised{\alpha_1}+2)}}\frac{1}{n^{\revised{\alpha_1}}}\leq \frac{C}{n^\revised{\min\{\alpha_1,\alpha_2\}}},
$$
where $C$ is defined by 
\begin{align}\label{constant C in first convergence rate}
C:=\exp\parens{T\parens{FKh_T\revised{L_\sigma}}^L}\parens{A_2\sqrt{\frac{F}{\revised{2\alpha_2+1}}}+LK\norm{\tX}_{C([0,T];L^2(I;\mathbb{R}^{1\times F}))} A_1\sqrt{\frac{2^{2\revised{\alpha_1}+2}-2}{(2\revised{\alpha_1}+1)(2\revised{\alpha_1}+2)}}}.
\end{align}
This completes the proof of \eqref{eq: rate of weighted graph}. The estimate \eqref{weighted graph: size transferability bounds} can be immediately obtained from \eqref{eq: rate of weighted graph} and the triangle inequality. 
\end{proof}

\begin{lemma}\label{lemma: best constant approximation}
    Suppose that $\Omega\subset\mathbb{R}^d$ and $f\in L^2(\Omega)$. Let $|\Omega|$ be the volume of $\Omega$. Then the constant $h:=\frac{1}{|\Omega|}\int_\Omega f(u)du$, is the best constant approximation of $f$, i.e., $\inf\{\|f-c\|_{L^2(\Omega)}:c\in\mathbb{R}\}=\left\|f-h\right\|_{L^2(\Omega)}$. 
\end{lemma}
\revised{
\begin{proof}
    We define a function $J(c):=\norm{f-c}_{L^2(\Omega)}^2$, $c\in\mathbb{R}$. Note that $J$ is a convex quadratic function, and its minimum is attained at $c=h$. 
\end{proof}
}
\begin{proof}[Proof of Theorem \ref{theorem: rate of simple graph}]

    We begin with estimating $\|\tW-\tW_n\|_{L^2({I^2})}$ \revised{using the box-counting dimension argument from Theorem~4.1 of \citet{medvedev2014nonlinear}}. Recall that $\mathcal{N}_\delta(\partial \tW^+)$ denotes the number of $\delta$-mesh cubes that intersect the boundary $\partial \tW^+$. We set $\delta=1/n$. Recall that $\tW_n$ is defined by \eqref{induced graphon Wn} with the adjacency matrix generated by \eqref{coeff W for simple graphs}. \revised{Observe that if a mesh cube is fully contained within the support $\tW^+$ (where $\tW(u,v)=1$) or fully outside it (where $\tW(u,v)=0$), then $\tW_n$ and $\tW$ are identical on that cube. Therefore, the function values of $\tW_n$ and $\tW$ can only differ over cubes that intersect $\partial \tW^+$. The number of such cubes is $\mathcal{N}_{1/n}(\partial \tW^+)$, and the area of each cube is $1/n^2$. This implies}
\begin{equation}\label{esitimate W-Wn third way}
        \|\tW-\tW_n\|_{L^2({I^2})}^2=\int_I |\tW(u,v)-\tW_n(u,v)|^2dudv\leq \mathcal{N}_{1/n}(\partial \tW^+)\frac{1}{n^2}. 
    \end{equation}
    According to definition \eqref{definition of dim box counting} of upper box-counting dimension, for any $\epsilon\in(0,2-b)$, there exists $N_{\epsilon,\tW}\in\mathbb{N}$ such that when $n>N_{\epsilon,\tW}$, $\frac{\log \mathcal{N}_{1/n}(\partial \tW^+)}{-\log(1/n)}<b+\epsilon$. Therefore, $\mathcal{N}_{1/n}(\partial \tW^+)\leq n^{b+\epsilon}$, which combining with \eqref{esitimate W-Wn third way} yields    
    \begin{equation}\label{esitimate of W-Wn third way final}
        \|\tW-\tW_n\|_{L^2({I^2})}\leq n^{-(1-\frac{b+\epsilon}{2})}.
    \end{equation}
    We next estimate $\|\tZ-\tZ_n\|_{L^2(I;\mathbb{R}^{1\times F})}$. Recall that $\tZ_n$ is the induced graphon feature function associated with the graph feature matrix generated in the way of \eqref{coeff G for simple graphs}. Let $\tZ_n'$ be the induced graphon feature function associated with the graph feature matrix generated in the way of \eqref{coeff G for weighted graphs}. It has been shown in the proof of Theorem \ref{theorem: rate of Lipschitz} that, with assumption {\color{black}\hyperref[AS3]{AS3}}, $
    \|\tZ-\tZ_n'\|_{L^2(I;\mathbb{R}^{1\times F})}\leq A_2\sqrt{\frac{F}{\revised{2\alpha_2+1}}}\frac{1}{n^{\revised{\alpha_2}}}$. According to Lemma \ref{lemma: best constant approximation}, we know that $
    \|\tZ-\tZ_n\|_{L^2(I;\mathbb{R}^{1\times F})}\leq \|\tZ-\tZ_n'\|_{L^2(I;\mathbb{R}^{1\times F})}$. Therefore,
    \begin{equation}\label{final norm of g - gn in the proof of simple graph finally}
    \|\tZ-\tZ_n\|_{L^2(I;\mathbb{R}^{1\times F})}\leq A_2\sqrt{\frac{F}{\revised{2\alpha_2+1}}}\frac{1}{n^{\revised{\alpha_2}}}. 
    \end{equation}
    With a similar argument in the proof of Theorem \ref{theorem: rate of Lipschitz}, by Theorem \ref{theorem: X-Xn leq G-Gn + TW-TWn} and estimates \eqref{esitimate of W-Wn third way final} and \eqref{final norm of g - gn in the proof of simple graph finally}, we have 
    $$
    \|\tX_n-\tX\|_{C([0,T];L^2(I;\mathbb{R}^{1\times F}))} \leq PA_2\sqrt{\frac{F}{\revised{2\alpha_2+1}}}\frac{1}{n^{\revised{\alpha_2}}}+Q  n^{-(1-\frac{b+\epsilon}{2})}\leq \frac{\widetilde{C}}{n^{\revised{\min\{1-\frac{b+\epsilon}{2},\alpha_2\}}}},
    $$
    where $\widetilde{C}$ is defined by 
\begin{equation}\label{tilde C in rate of simple graphs}
    \widetilde C:=\exp\parens{T\parens{FKh_T\revised{L_\sigma}}^L}\parens{A_2\sqrt{\frac{F}{\revised{2\alpha_2+1}}}+LK\norm{\tX}_{C([0,T];L^2(I;\mathbb{R}^{1\times F}))}}. 
\end{equation}
This proves \eqref{eq: rate of unweighted graph}. The estimate \eqref{unweighted graph: size transferability bounds} can be obtained from \eqref{eq: rate of unweighted graph} and the triangle inequality.
\end{proof}

\section{Details of Numerical Experiments}\label{appendix: numerical experiments}

\subsection{Graphon Convergence Rates}\label{appendix: additionalgraphons}
\paragraph{Graphons} We include three graphon experiments to \revised{complement our study in Section} \ref{sec:graphon_conv_rate_experiment} and further verify our main results. We utilize one additional weighted graphon, an extremely oscillatory Lipschitz graphon defined by:
\begin{equation}\label{def: irreg graphon}
\tW(u,v) = \frac 1 2 \left( 1 + \sin(20\pi x) \sin(20 \pi y) \right).
\end{equation}
We also experiment with two additional \(\{0,1\}\)-valued graphons. We create a checkerboard graphon with box-counting dimension $1$ but with extremely similar structure to the oscillatory Lipshitz graphon, and a Sierpiński carpet fractal \citep{sierpinski_carpet_citation} with box-counting dimension about $1.89$. We illustrate these graphons in Figure \ref{fig:three_graphon_extra}.

\begin{figure}[H]
    \centering
    \includegraphics[width=\textwidth]{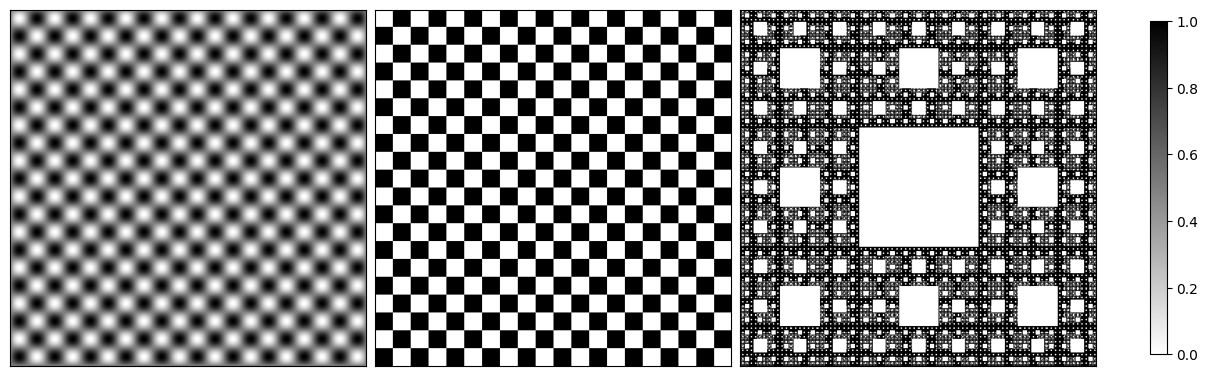}
    \caption{Oscillatory Lipschitz (left), Checkerboard (center), and Sierpinski (right) graphons.}
    \label{fig:three_graphon_extra}
\end{figure}

\paragraph{Additional Experiment Details}  

For each graphon considered, we conduct 100 trials with independent random initializations, including both random weight initialization of the GNDE model and random input features. \revised{Specifically, we sample $\{a_k\}$ and $\{b_k\}$ i.i.d.\ from the uniform distribution on $[-1,1]$ and generate initial conditions using random Fourier polynomials of degree $D = 10$, defined by $\tZ(u) := \sum_{k=1}^{D} a_k \cos(2 \pi k u) + b_k \sin(2 \pi k u)$.} \newrevised{We compute the slope of the log-log line of best fit to approximate rate of convergence for each individual configuration and} report the mean and standard deviation in Figures~\ref{fig:error_bars_graphon_main} and \ref{fig:error_bars_graphon_additional}. For the Hölder-$\tfrac{1}{2}$ graphon case, we take the initial feature function as $\tZ(u)=\sum_{k=1}^{D}a_k\cos(2\pi b_k u)$ with $a_k=a^k$ and $b_k=b^k$ with $a=1/\sqrt{b}$ and $a$ sampled i.i.d.\ from $[3,10]$. \revised{This construction approximates an initial feature function that is Hölder-$\tfrac{1}{2}$. While the finite truncation is smooth, for sufficiently large $D$ it exhibits the characteristic roughness of Hölder-$\tfrac{1}{2}$ functions with fast-growing derivatives.} \newrevised{This altered feature function allows empirical verification of Theorem \ref{theorem: rate of Lipschitz} when the Hölder constants of both the graphon and the feature function are no longer $1$.}

All experiments were carried out locally on 4 Nvidia A4000 GPUs. As there is no training step, experiment runtimes are fast.

\begin{figure}[htbp]
    \centering
    \includegraphics[width=0.5\textwidth]{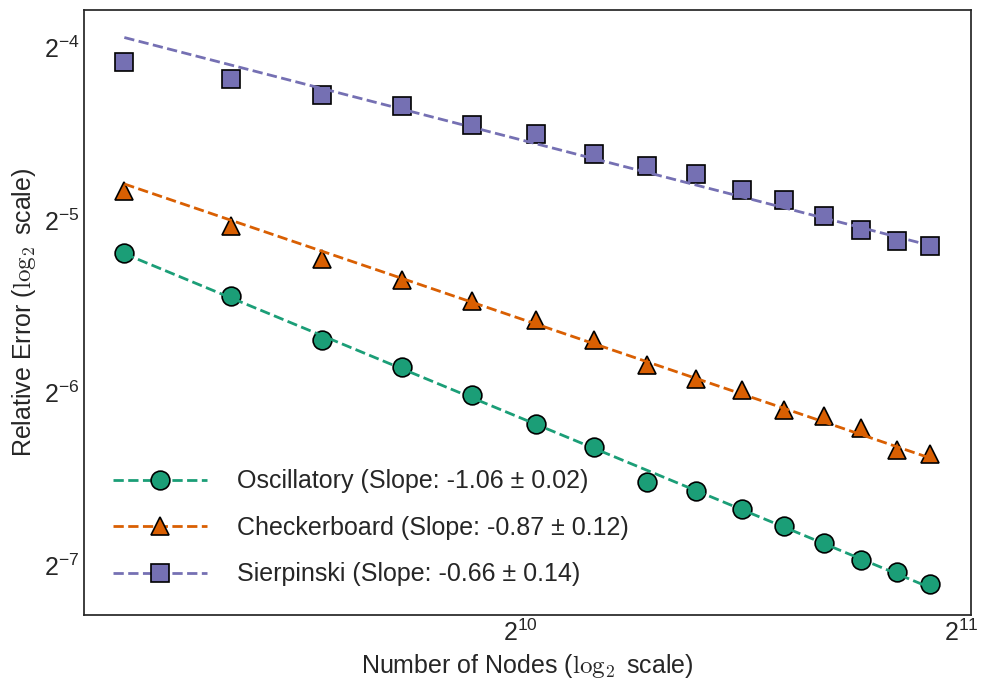}
    \caption{\textbf{Convergence rates of GNDE solutions.} Relative errors between GNDE and Graphon-NDE solutions on graphs sampled from the three additional graphons: (1) \emph{Oscillatory Lipschitz graphon}, matching the expected \(\mathcal{O}(1/n)\) rate, (2) \emph{checkerboard graphon} (box counting dimension 1) with slower observed rate despite similarity to the Oscillatory Lipschitz graphon and (3) \emph{Sierpiński carpet graphon} (fractal boundary with box counting dimension 1.89). The checkerboard graphon yields faster convergence than the Sierpiński carpet graphon, again consistent with the trend indicated in Theorem~\ref{theorem: rate of simple graph}.}
    \label{fig:multi_graph_experiment_extra}
\end{figure}

\begin{figure}[H]
    \centering
    \includegraphics[width=\textwidth]{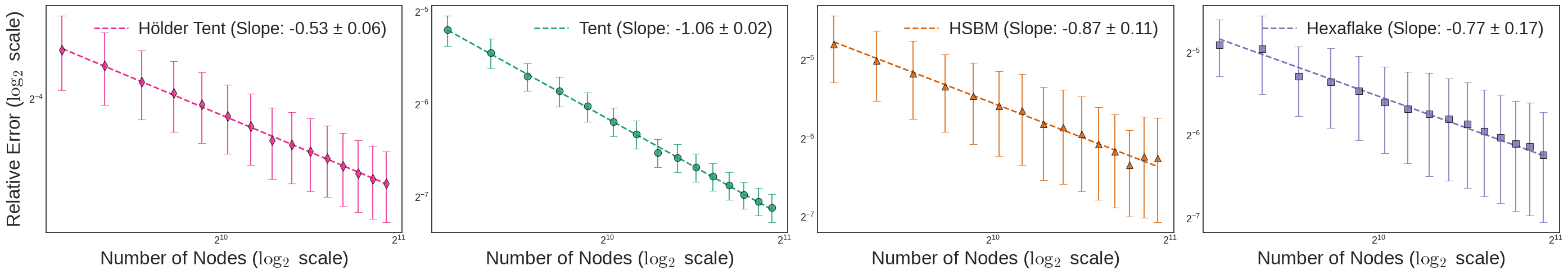}
    \caption{Hölder Tent (left), Tent (center-left), HSBM (center-right), and Hexaflake (right) graphon convergence with error bars displayed.}
    \label{fig:error_bars_graphon_main}
\end{figure}

\begin{figure}[H]
    \centering
    \includegraphics[width=\textwidth]{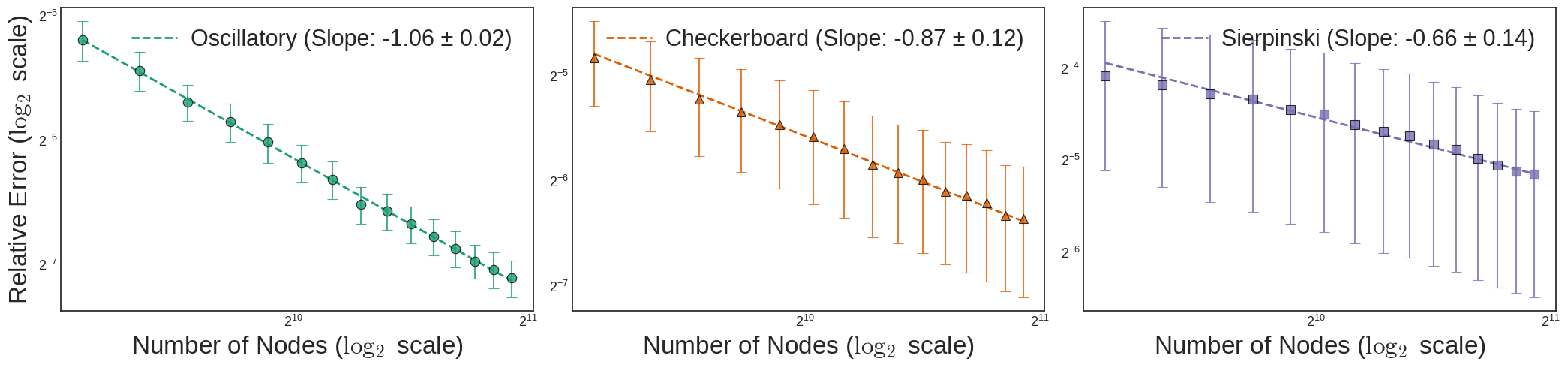}
    \caption{Oscillatory Lipschitz (left), Checkerboard (center), and Sierpiński (right) graphon convergence with error bars displayed.}
    \label{fig:error_bars_graphon_additional}
\end{figure}

\paragraph{Analysis}  

Our Lipschitz graphon rate continues to consistently match \(\mathcal{O}(1/n)\) regardless of the relative complexity of the graphon function used. For our $\{0,1\}$-valued graphons, we see the \revised{checkerboard graphon, even though structurally similar to the Lipschitz graphon,} converges slower with rate $\mathcal{O}(1/n^{0.87})$, empirically verifying the meaningful divergence between the two cases. There is relatively low variance from the mean for each of the Lipschitz graphons, but relatively high variance for each of the $\{0,1\}$-valued graphons, mainly due to the hard problem of sampling \newrevised{from binary graphons whose support has a fractal boundary} resulting in several outliers over the $100$ trials. \revised{The observed high variance underscores the importance of our theoretical analysis, as theoretical worst-case guarantees are particularly useful in numerically unstable regimes.}

\section{Real Data Node Classification: Additional Details}\label{appendix: real data details}

\paragraph{Dataset Statistics} We experiment with a variety of the most popular graph node classification datasets, including homophilic and heterophilic datasets of various sizes. We adapt the literature standard split configurations for each dataset: for the heterophilic datasets, we adopt the 60/20/20 random splits of \citet{pei2020geomgcn}; for the citation networks, we use the Planetoid splits of \citet{yang2016revisiting}; and for the ogbn-arxiv dataset, we use the fixed splits of \citet{hu2021opengraphbenchmarkdatasets}. The comprehensive dataset statistics and split configurations are in Tables \ref{table:dataset_structure} and \ref{table:dataset_splits}.

\begin{table}[H]
\centering
\begin{tabular}{ |c|c|c|c|c|c| }
 \hline
 \textbf{Dataset Name} & \textbf{Nodes} & \textbf{Edges} & \textbf{Features} & \textbf{Classes} & \textbf{Homophily}\\
 \hline
Actor & 7,600 & 30,019 & 932 & 5 & 0.2188 \\
Chameleon & 2,277 & 36,101 & 2,325 & 5 & 0.2350 \\
Cornell & 183 & 298 & 1,703 & 5 & 0.1309\\
 Citeseer & 3,327 & 9,228 & 3,703 & 6 & 0.7391\\
 Cora & 2,708 & 10,556 & 1,433 & 7 & 0.8100 \\
 Pubmed & 19,717 & 88,651 & 500 & 3 & 0.8024\\
ogbn-arxiv & 169,343 & 2,332,486 & 128 & 40 & 0.6551 \\
Squirrel & 5,201 & 217,073 & 2,089 & 5 & 0.2239 \\
Texas & 183 & 325 & 1,703 & 5 & 0.1077\\
Wisconsin & 251 & 515 & 1,703 & 5 & 0.1961\\
 \hline
\end{tabular}
\vspace{6pt}
\caption{Dataset statistics.}
\label{table:dataset_structure}
\end{table}

\begin{table}[htbp]
\centering
\begin{tabular}{ |c|c|c|c| }
 \hline
 \textbf{Dataset Name} & \textbf{Training} & \textbf{Validation} & \textbf{Testing}\\
 \hline
Actor & $60\%$ & $20\%$ & $20\%$ \\
Chameleon & $60\%$ & $20\%$ & $20\%$ \\
Cornell & $60\%$ & $20\%$ & $20\%$ \\
 Citeseer & 120 & 500 & 1000 \\
 Cora & 140 & 500 & 1000 \\
 Pubmed & 60 & 500 & 1000 \\
ogbn-arxiv & 90,941 & 29,799 & 48,603 \\
Squirrel & $60\%$ & $20\%$ & $20\%$ \\
Texas & $60\%$ & $20\%$ & $20\%$ \\
Wisconsin & $60\%$ & $20\%$ & $20\%$ \\
 \hline
\end{tabular}
\vspace{6pt}
\caption{Dataset split configurations.}
\label{table:dataset_splits}
\end{table}

\paragraph{Hyperparameter Selection and Model Architecture}\label{sec:model_architecture}

Model training hyperparameters were selected through a grid search, optimizing performance on the full Cora dataset as the evaluation metric. The same hyperparameters were consistently applied across all datasets and subgraph sizes to ensure fairness and comparability. We employed the Adam optimizer with default hyperparameters $\beta_1 = 0.9$ and $\beta_2 = 0.999$. To enhance regularization and convergence, we incorporated a dropout ratio \citep{srivastava2014dropout} and a weight decay parameter \citep{krogh1991simple}, as detailed in Table \ref{table:hyperparameters}. We utilized the classical fourth-order Runge-Kutta solver \citep{runge1895numerische, butcher2008numerical} for all GNDE evaluations, as it provides a favorable balance between computational efficiency and accuracy.

In all cases, training was performed over $3000$ epochs, with light early stopping criteria in place to mitigate overfitting; training was terminated early in the rare cases where model validation loss did not drop for 500 consecutive epochs. Early stopping triggered most frequently for smaller training subgraphs, where the validation loss tends to plateau before the model has fully converged; at the 10\% subgraph level roughly half of runs terminated early, falling to under 5\% by the 30\% subgraph and under 1\% beyond that. Runs that triggered early stopping produced test accuracies a few percentage points below those that completed full training, consistent with the model not having fully converged on very small subgraphs. Because early stopping incidence is negligible for subgraphs at or above 30\% of the full graph, it does not materially affect the reported results. After training, the model was transferred to the full graph and evaluated. Twenty random sequences of subgraphs were constructed for each dataset. On every such sequence, twenty randomly initialized models were trained on each subgraph. Results reported are the mean and standard deviation over all trials and weight initializations. All experiments were performed locally on 4 Nvidia A4000 GPUs.

\begin{table}[ht]
\centering
\begin{tabular}{ |c|c|c| }
 \hline
 \textbf{Hyperparameter Name} & \textbf{Grid Search Choices} & \textbf{Final Choice} \\
 \hline
 Learning Rate & $\{10^{-2}, 10^{-3}, 10^{-4}, 10^{-5}\}$ & $10^{-3}$ \\
 Weight Decay & $\{5 \cdot 10^{-4}, 10^{-4}, 5 \cdot 10^{-5}\}$ & $5 \cdot 10^{-4}$ \\
 GNN Head Dropout & $\{0.2, 0.4, 0.6\}$ & 0.4 \\
 GNDE Hidden Features & $\{16, 32, 64\}$ & 64 \\
 \hline
\end{tabular}
\vspace{6pt}
\caption{Hyperparameters used for model training, including grid search choices and final values.}
\label{table:hyperparameters}
\end{table}

\begin{table}[ht]
\centering
\begin{tabular}{ |c|c|c|c|c|c| }
 \hline
 \textbf{Model Part} & \textbf{(Input, Hidden, Output) Features} & \textbf{$L$} & \textbf{$K$} & \textbf{Activation} & \textbf{Dropout} \\
 \hline
 GNN Head & (Varied, $64$, $64$) & $1$ & $2$ & ReLU & $0.4$ \\
 GNDE & (64, 64, 64) & $2$ & $2$ & ReLU & $0.9$ \\
 GNN Tail & ($64, 64$, Varied) & $1$ & $1$ & None & $0$ \\
 \hline
\end{tabular}
\vspace{6pt}
\caption{Architecture details for the GNN Head, GNDE, and GNN Tail components of the model.}
\label{table:architecture}
\end{table}

\end{document}